%% file: main.tex
\documentclass{article}

\usepackage{microtype}
\usepackage{graphicx}
\usepackage{subfigure}
\usepackage{booktabs} %

\usepackage{hyperref}

\usepackage[accepted]{icml2023}
\usepackage{templates}

\icmltitlerunning{On Regularization and Inference with Label Constraints}

\begin{document}

\twocolumn[
\icmltitle{
    \vspace*{-0.5in}
    {{\small \hfill ICML'2023 }\\
    \vspace*{.25in}}
    On Regularization and Inference with Label Constraints
}

\icmlsetsymbol{equal}{*}

\begin{icmlauthorlist}
\icmlauthor{Kaifu Wang}{upenn}
\icmlauthor{Hangfeng He}{rochester}
\icmlauthor{Tin D. Nguyen}{mit}
\icmlauthor{Piyush Kumar}{str}
\icmlauthor{Dan Roth}{upenn}
\end{icmlauthorlist}

\icmlaffiliation{upenn}{University of Pennsylvania, Philadelphia, PA, USA}
\icmlaffiliation{mit}{Massachusetts Institute of Technology, Cambridge, MA, USA}
\icmlaffiliation{str}{Systems and Technology Research, Woburn, MA USA}
\icmlaffiliation{rochester}{University of Rochester, Rochester, NY, USA (Part of the work done while at the University of Pennsylvania.)}

\icmlcorrespondingauthor{Piyush Kumar}{piyush.kumar@str.us}
\icmlcorrespondingauthor{Dan Roth}{danroth@seas.upenn.edu}

\icmlkeywords{Machine Learning, ICML}

\vskip 0.3in
]

\printAffiliationsAndNotice{}

\begin{abstract}
    Prior knowledge and symbolic rules in machine learning are often expressed in the form of label constraints, especially in structured prediction problems.
    In this work, we compare two common strategies for encoding label constraints in a machine learning pipeline, \emph{regularization with constraints} and \emph{constrained inference}, by quantifying their impact on model performance.
    For regularization, we show that it narrows the generalization gap by precluding models that are inconsistent with the constraints. However, its preference for small violations introduces a bias toward a suboptimal model.
    For constrained inference, we show that it reduces the population risk by correcting a model's violation, and hence turns the violation into an advantage.
    Given these differences, we further explore the use of two approaches together and propose conditions for constrained inference to compensate for the bias introduced by regularization, aiming to improve both the model complexity and optimal risk.
\end{abstract}

\input{sections/introduction.tex}

\input{sections/prelim.tex}
\input{sections/learning.tex}
\input{sections/inference.tex}
\input{sections/combine.tex}

\input{sections/experiments.tex}
\input{sections/related-work.tex}
\input{sections/conclusion.tex}

\section*{Acknowledgements}

This work was partially supported by Contract FA8750-19-2-0201 with the US Defense Advanced Research Projects Agency (DARPA). Approved for Public Release, Distribution Unlimited. The views expressed are those of the authors and do not reflect the official policy or position of the Department of Defense or the U.S. Government.

This work was also partially sponsored by the Army Research Office and was accomplished under Grant Number W911NF-20-1-0080. The views and conclusions contained in this document are those of the authors and should not be interpreted as representing the official policies, either expressed or implied, of the Army Research Office or the U.S. Government. The U.S. Government is authorized to reproduce and distribute reprints for Government purposes notwithstanding any copyright notation herein.

This work was also partially funded by ONR Contract N00014-19-1-2620. 

\bibliography{ccg, cited, new}
\bibliographystyle{icml2023}

\newpage
\input{sections/appendix.tex}

\end{document}

%% file: sections/introduction.tex
\section{Introduction}

Domain knowledge in machine learning is often framed as constraints on the output label space.
Such label constraints have been widely identified in natural language processing tasks 
\citep{RothYi04, 
MartinsSmXi09a, 
NFWR18, 
lu-etal-2022-neurologic} 
and studied in the context of structured prediction
\citep{PRYZ05, 
ChangRaRo07, 
CRRR08, 
GGGT10, 
ChangRaRo12, 
PanMeSr20}.
For example, in temporal reasoning \citep{NFWR18} where the model is asked to label the relations (``before'' or ``after'') among a set of events, the assigned labels will need to satisfy a transitivity constraint which means, for example, the facts that an event $E_1$ is after $E_2$ and that $E_2$ is after $E_3$ imply that $E_1$ is after $E_3$.
The central question is how to encode such a constraint into a learning algorithm to ensure better performance and generalization of the learned model.

Practitioners have developed two techniques to encode a label constraint in a machine learning pipeline. The first, called \emph{regularization with constraints}, penalizes a model for its violation of the constraint in addition to the classification loss \citep{GGGT10, Li2019ALF}. The second, called \emph{inference with constraints}, modifies prediction rules directly by enforcing strictly constrained inference \citep{RothYi04, scholak-etal-2021-picard} or balancing the original model's output with the constraint in a soft way \citep{CRRR08, ChangRaRo12}.

Although these two learning algorithms have been shown to be empirically successful, we are not aware of theoretical analyses that elucidate each algorithm's advantages or disadvantages in comparison with the other one. Natural questions include, how do these two differ in their impact on the learned model? Moreover, in practice, the constraints could be noisy i.e.\ \citep{hu-etal-2016-harnessing, WZCR21}. In such cases, do they still improve the model performance? If so, by how much?

Focusing on multiclass classification with label constraints, we compare regularization with constraints and constrained inference.
For each algorithm, we quantify its optimal risk (aka approximation error) and its generalization gap (aka estimation error).
Specifically, in Section \ref{learning}, we show that regularization with constraints achieves a smaller generalization error by reducing the model complexity but will introduce a bias towards a suboptimal model if the risk minimizer and violation minimizer does not coincide.
In Section \ref{inference}, we study a broad family of constrained inference model called Constrained Conditional Model (CCM) \citep{CRRR08, ChangRaRo12} and point out that the constrained inference could reduce the risk of a model if and only if the model violates the constraint more than the true data distribution.
This further suggests finding models with higher violation, which contrasts the learning objective used in regularization that discourages violation.
Given these contrasts, we further study the combination and interaction of the two methods in Section \ref{combo} and describe how constrained inference could compensate for the bias introduced by regularization.

To the best of our knowledge, our analysis is the first to provide a theoretical view on comparing the two approaches. We believe in the importance of this comparison and hope to bring this problem to the attention of the machine learning community.
In summary, our contributions include:
\begin{enumerate}
    \item 
    We provide an error bound (Theorem \ref{bound-regularization}) that describes the tradeoff between the generalization gap and the optimal risk when performing regularization with constraints.

    \item 
    We propose a sufficient and necessary condition (Theorem \ref{change-of-risk}) for constrained inference to improve a model by quantifying its reduction in risk.
    Based on this, we further argue that constrained inference, when used at training time, implicitly modifies the training objective in an opposite direction as in the regularization approach (Proposition \ref{on-training}).

    \item 
    We study the combination of regularization and constrained inference, and propose sufficient (Theorem \ref{combo-help}) as well as necessary (Theorem \ref{combo-not-help}) conditions for the combined algorithm to achieve improvement in both optimal risk and model complexity.
    
\end{enumerate}

Proofs of all the theoretical results are in the appendix.

%% file: sections/prelim.tex
\section{Preliminaries}

Our goal is to learn a mapping from the instance space $\+X$ to the output space $\+Y$. 
The learner has access to a set of labeled training data $S_\:L$ of size $m_{\:L}$, which contains i.i.d. samples of a distribution $\-P$ on $\+X \times \+Y$. 
The marginal distribution of $X$ is denoted as $\-P_X$.
In this work, we assume the ground truth label associated with $x \in \+X$ is generated by a deterministic mapping $y_\ora:\+X \rightarrow \+Y$ ($_\ora$ is short for oracle). We also denote the true label as $y_\ora$ when the context is clear.

\paragraph{Model.}
The scoring class $\+F$ contains scoring functions $f:\+X \times \+Y \rightarrow \-R$. 
We will also call a $f\in\+F$ a classifier.
Let $\Delta_{\+Y}$ be the $|\+Y|$-dimensional probability simplex. Each scoring function 
induces a probabilistic prediction $\-P_f(\cdot|x) \in \Delta_{\+Y}$ by performing \emph{softmax inference} as $\-P(y|x) \propto \exp(f(x,y))$. 

\paragraph{Loss Function.}
The prediction of $f$ at $x$ is evaluated by the \emph{classification  error} (or $\ell^1$ loss) $L(x,y_\ora,f) := 1 - \-P_f(y_\ora|x)$, which is half the $\ell^1$ distance the between the one-hot distribution $e_{y_\ora}$ and $\-P_f$ on $\Delta_\+Y$.
It can also be viewed as a smoothed version of the standard zero-one loss in the sense that $\lim_{t \rightarrow \infty} L(x,y_\ora,tf) = \-1\{\argmax_{y\in\+Y}f(x,y) \ne y_\ora\}$.
More background on the definition of the $\ell^1$ loss are provided in Appendix \ref{loss}.
A scoring function $f$ is evaluated by its \emph{risk} $R(f) := \-E[L(x,y_\ora,f)]$. The empirical estimate of the risk using the labeled examples in $S_{\:L}$ is denoted as $\widehat{R}(f, S_\:L)$. 
We also consider the cross-entropy surrogate loss defined as $L_\ce (x,y_\ora,f) := -\log \-P_f(y_\ora|x)$ and refer its expectation $R_\ce(f) = \-E[L_\ce(x,y_\ora,f)]$ as cross-entropy risk.

\paragraph{Label constraint.}
A \emph{label constraint} (or \emph{constraint} for short) is a deterministic mapping $C:\+X \rightarrow 2^\+Y-\{\emptyset\}$. Namely, $C$ maps an instance $x$ to a nonempty subset of $\+Y$, which may or may not contain the true label $y_\ora(x)$. In particular, we say a constraint $C$ is \emph{noise-free} if $\-P(y_\ora\in C(x))=1$. Otherwise, $C$ is said to be a \emph{noisy} constraint and its noise rate is denoted as $V_\ora := \-P(y_\ora(x) \notin C(x))$.

\paragraph{Violation.}
A constraint $C$ is equipped a \emph{violation function}, which is an indicator function $v_C(x,y) = \-1\{y\notin C(x)\}$. We also overload the notation $v$ and define the violation of a classifier $f$ at an instance $x$ as $v_C(x,f):= 1-\-P_f(C(x)|x) = \sum_{y\notin C(x)} \-P_f(y|x)$. Its expectation is $V_C(f):= \-E[v_C(x,f)]$. We elide the subscript $C$ and write them as $v(x,y), v(x,f)$ and $V(f)$ when the context is clear. Similar to the classification error, we consider a cross-entropy surrogate of the violation function defined as $v_\ce(x,f):=-\log\-P_f(C(x))$ and its expectation $V_\ce(f) = \-E[v_\ce(x,f)]$.

\paragraph{Rademacher complexity.}
We use the following version of Rademacher complexity that is adopted from \citet{CKMY16} to characterize the generalization ability of the scoring space of multiclass classifiers $\+F$:

\begin{definition}[Rademacher complexity of $\+F$] 
    The \emph{empirical Rademacher complexity} of scoring class $\+F$ with respect to a set $S = \{x_i\}_{i=1}^m$ that contains $m$ samples of the instance is defined as
    \begin{equation}
        \label{empirical-factor-graph-complexity}
        \begin{aligned}
            \widehat{\mathfrak{R}}_m(\+F;S)
            := 
            \frac{1}{m}
            \-E_\epsilon\left[
                \sup_{f\in \+F} 
                \sum_{i=1}^m 
                \sum_{y\in \+Y} 
                \epsilon_{i,y} f(x_i,y)
            \right]
        \end{aligned}
    \end{equation}
    where $\epsilon=(\epsilon_{i,y})_{i\in [m],y\in \+Y}$ are independent Rademacher random variables, each of which is uniformly distributed over $\{-1,+1\}$. The \emph{Rademacher complexity} of scoring class $\+F$ is the expectation of the empirical version:
    \begin{equation}
    \label{factor-graph-complexity}
    \begin{aligned}
        {\mathfrak{R}}_m(\+F)
        := \-E_{S \sim \-P_X^m}[\widehat{\mathfrak{R}}_m(\+F;S)]
    \end{aligned}
    \end{equation}
\end{definition}

This definition of Rademacher complexity is a special case of the \emph{factor graph complexity} proposed by \citet{CKMY16}, which is defined for more general structured prediction models. It is hence possible to extend our results of the generalization bounds to structured models by replacing the Rademacher complexity with factor graph complexity. In this work, we focus on multiclass classifiers for the simplicity of presentation.

%% file: sections/learning.tex
\section{Regularization with Constraints}
\label{learning}

In a standard machine learning algorithm, the learner receives a set of labeled data $S_\:L \in \cup_{m=1}^\infty(\+X \times \+Y)^m$ and finds the \emph{empirical risk minimizer}, which is defined as $\argmin_{f \in \+F} \hat{R}(f;S_\:L)$.
In this section, we consider a method that modifies this learning objective by adding a regularization term defined with the constraint $C$. Precisely, we consider minimizing an augmented objective defined as
$
\+L_\rho (f)
:= R(f) + \rho V(f) 
$
where $\rho \ge 0$ is a fixed tradeoff parameter. 

The idea of regularizing the model by adding a penalty for the violation of the constraints on an unlabeled dataset is widely adopted in the literature. In particular, the cross entropy violation is known as the \emph{semantic loss} \citep{pmlr-v80-xu18h} in the context of logical constraints. Other designs of the regularization term include using the KL-divergence on the probability space in the \emph{posterior regularization} algorithm \citep{GGGT10} and using the $t$-norms from fuzzy logic \citep{Li2019ALF}.

We will show this algorithm improves the generalization error by reducing the complexity of the scoring space (Theorem \ref{bound-regularization}), but in general leads to a larger classification risk in the long run (Proposition \ref{vio-risk-inequalities}), thus resulting in a tradeoff between estimation and approximation errors. 

\subsection{Semi-supervised Regularization with Constraints}
\label{regularization}

We consider a semi-supervised approach where the learner has access to an unlabeled dataset $S_\:U$ that contains $m_\:U$ independent samples of the instance $X$, resulting in the following definition.

\begin{definition}[ERVM] Given a labeled dataset $S_\:L$ of size $m_\:L$ and an unlabeled dataset $S_\:U$ of size $m_\:U$, a scoring space $\+F$ and a tradeoff parameter $\rho \ge 0$, we define and denote the \emph{empirical risk and violation minimizer} (ERVM) as:
    \begin{equation}
    \begin{aligned}
        \widehat{f}_\rho(S_\:L,S_{\:U})
        := \argmin_{f\in \+F} &
            \left(
                \frac{1}{m_\:L} \sum_{(x,y)\in S_\:L} L(x,y,f)  \right. \\ 
        & \quad \left. + \frac{\rho}{m_\:U} \sum_{x\in S_\:U} v_C(x,f)
            \right).
    \end{aligned}
    \end{equation}
    We also denote the expected version as:
    \begin{equation}
    \label{regularized-objective}
    \begin{aligned}
        f_{\rho}
        & := \argmin_{f \in \+F} R(f) + \rho V_C(f).
    \end{aligned}
    \end{equation}
\end{definition}

For example, with our notation, $\hat{f}_0$ is the ERM and $f_\infty$ is the minimizer of the expected violation function. Notice that the minimizer in general is non-unique. Therefore, when we state any proposition that is related to $f_\rho$ or $\hat{f}_\rho$, we mean the proposition will hold for any of the minimizers.

\subsection{Deviation from The Optimal Risk}

In this section, we study how the risk of the minimizer $f_\rho$ will deviate from the optimal risk in $\+F$. The reason that we are interested in bounding $R(f_\rho)$ is that in general the minimizer $R(f_\rho)$ is non-unique and may have different values of risks. Therefore, to describe the risk of ERVM in the long run (in Theorem \ref{gen-bound}), we provide an upper bound for all the possible risks of $f_\rho$.

\begin{proposition}[Deviation from the optimal risk] 
    \label{vio-risk-inequalities}
    For any constraint $C$ and $\rho \ge 0$, the following holds.
            \begin{equation}
            \begin{aligned}
                {R}(f_0)
                \le {R}(f_\rho) 
                \le {R}(f_0) + \rho ({V}(f_0) - {V}(f_\infty)) 
                .
            \end{aligned}
            \end{equation}
            The same relation also holds for the empirical estimates $\hat R$ and $\hat V$. Moreover, for any $\rho>0$, there exists a scoring space and data distribution so that the RHS can be reached even with a noise-free constraint $C$.
    
\end{proposition}

This result shows the minimizer of the regularized objective in general has a suboptimal risk over $\+F$. On the other hand, if the risk minimizer is simultaneously a violation minimizer, i.e., ${V}(f_0) = {V}(f_\infty)$, this relation implies consistency, i.e., $R(f_\rho) = R(f_0)$.
This quantity $V(f_0)$ can be small when the noise rate $V_\ora$ is small and the model is expressive enough (e.g., a deep neural net) to approximate the true model.

\subsection{Generalization Bounds}

Now we discuss how regularization could reduce the complexity of the hypothesis class. The first step is to show that the violation of the target hypothesis is not too large. In particular, the following bound is a direct consequence of minimizing the regularized objective:
\begin{lemma}[Regularization implies small violation] 
    \label{f-rho-violation}
    Let $f_\rho$ be the regularized learning objective defined as in \eqref{regularized-objective}. If the minimum violation in $\+F$ is upper bounded by a known constant $u \ge 0$, i.e., $V(f_\infty) \le u$, then $V(f_\rho) \le 1/\rho + u$.
\end{lemma}

The upper bound $u$ can be set to arbitrarily small by adding a baseline model defined as $f_t(x,y) = t\cdot \-1\{y\in C(x)\}$ and driving $t$ to infinite. This construction is possible due to the fact that the mapping $C$ is known to the learner. The benefits of knowing $C$ will be further explored in Section \ref{inference} when we discuss inference with constraints.

For any $B \ge 0$, we let $\+F_B := \{f \in \+F| V(f) \le B\}$ be the set of classifiers with small violation.
From the above discussion, we know that the target hypothesis $f_\rho$ will lie in a smaller space $\+F_{u+1/\rho}$, which is characterized by the violation function and hence can be identified only with unlabeled data. To this end, we describe how the violation as well as the risk can be estimated with data.

\begin{lemma}[Generalization gap of $\ell^1$ loss and violation]
    \label{gen-bound}
    Given a labeled dataset $S_\:L$ of size $m_\:L$, for any $\delta>0$, with probabilistic at least $1-\delta$, the following inequality holds uniformly for $f \in \+F$:

    \begin{equation}
        \begin{aligned}
            R(f) 
            \le \hat{R}(f;S_\:L) + \mathfrak{R}_{m_\:L}(\+F) + \sqrt{\frac{\log(1/\delta)}{2m_\:L}}
        \end{aligned}
    \end{equation}
    Given a unlabeled dataset $S_\:U$ of size $m_\:U$, for any $\delta>0$, with probabilistic at least $1-\delta$, the following inequality holds uniformly for $f \in \+F$:
    \begin{equation}
        \label{eqn:violation-gen}
        \begin{aligned}
            V(f) 
            \le \hat{V}(f;S_\:U) + \mathfrak{R}_{m_\:U}(\+F) + \sqrt{\frac{\log(1/\delta)}{2m_\:U}}
        \end{aligned}
    \end{equation}
\end{lemma}

The proof of this result relies on a contraction lemma established in \citet{CKMY16}, which was used to analyze the argmax inference with margin losses. Our analysis extends their results to softmax inference, which may be of independent interest. 

Furthermore, if the size of the constrained set $C(x)$ is a constant, namely $|C(x)|=c_0 < c = |\+Y|$ for all $x \in \+X$, then the Rademacher complexity term of equation \eqref{eqn:violation-gen} can be improved to $\frac{\sqrt{2}}{2}\sqrt{\frac{1}{c-c_0} + \frac{1}{c_0}} \mathfrak{R}_{m_\:U}(\+F)$ (see the discussion in the proof). 
This term is symmetric with the transformation $c_0 \mapsto c-c_0$, due to the fact that estimating the violation $V_C$ of a constraint $C$ is equivalent to estimating $V_{\+Y-C}$.
In particular, when $c_0 < c/2$, if the constraint is more restrictive and informative (so that $c_0$ is small), it can be more difficult to estimate the violation.

\begin{remark}[Relation to cross-entropy loss]
Assuming $\lim_{m 
\rightarrow \infty} \mathfrak{R}_m(\+F) = 0$, this result implies $\+L_\rho$ can be approximated by its empirical version $\hat{\+L}_\rho$ with sufficient amount of data. On the other hand, since $\hat{\+L}_\rho$ is upper bounded by its cross-entropy surrogate $\hat{R}_\ce + \rho \hat{V}_\ce$, we further have that
\begin{equation}
\begin{aligned}
    \+L_\rho(f) 
    \le \hat{R}_\ce(f,S_\:L) + \rho \hat{V}_\ce(f,S_\:U) + o_{m_\:L, m_\:U}(1)
\end{aligned}
\end{equation}
where $o_{m_\:L, m_\:U}(1)$ converges to 0 as $m_\:L, m_\:U \rightarrow \infty$.
Therefore, in practice one can  minimize this upper bound by solving the convex surrogate problem
\begin{equation}
\label{regularized-objective-ce}
\begin{aligned}
    \min_{f \in \+F} \hat{R}_\ce(f,S_\:L) + \rho \hat{V}_\ce(f,S_\:U).
\end{aligned}
\end{equation}
where $\hat{R}_\ce(f,S_\:L)$ and $V_\ce(f,S_\:U)$ are the empirical average of the cross-entropy loss and violation.
\end{remark}

Finally, using these results, we bound the risk of the classifier learned by ERVM. For simplicity, we will denote the generalization gap $B(\delta, m, \+F) := \mathfrak{R}_{m}(\+F) + 2\sqrt{\frac{\log(1/\delta)}{2m}}$.

\begin{theorem}[Error bound]
\label{bound-regularization}
    We have with probability at least $1-6\delta$ that
    \begin{equation}
    \label{err-bound}
        \begin{aligned}
            R(\hat{f}_\rho) 
            & \le R(f_0) + \rho V(f_0) - \rho V(f_\infty)\\
            & \quad + {\mathfrak{R}}_{m_\:L}(\+F_{1/\rho + u + B(\delta, m_\:U, \mathcal{F})}) \\ 
            & \quad + \rho {\mathfrak{R}}_{m_\:U}(\+F_{1/\rho + u + B(\delta, m_\:U, \mathcal{F})}) \\ 
            & \quad + 2 \sqrt{\frac{\log(2/\delta)}{2m_\:L}} + 2\rho \sqrt{\frac{\log(2/\delta)}{2m_\:U}}
        \end{aligned}
    \end{equation}
    where $\mathfrak{R}(\cdot)$ is the Rademacher complexity defined in \eqref{factor-graph-complexity}.
\end{theorem}

\begin{proofsketch}
    First, we show $\hat{f}_\rho$ and $f_\rho$ both lie in the subspace $\+F_{1/\rho + u + B(\delta, m_\:U, \mathcal{F})}$ with high probability since the violation can be well-approximated, according to Lemma \ref{gen-bound}. 
    Then, the gap between the objective $\+L(f_\rho)$ and $\+L(\hat{f}_\rho)$ is controlled by the  Rademacher complexity of $\+F_{1/\rho + u + B(\delta, m_\:U, \mathcal{F})}$. 
    Finally, we use the inequalities established in Lemma \ref{vio-risk-inequalities} to further upper bound the term $\+L(f_\rho)$ using the risk and violation of $f_0$.

    Using the same proof technique, this result can be extended to other choices of loss function as long as: 
    (a) The loss is bounded so that the optimal regularized model has a small violation, as in Lemma \ref{f-rho-violation}. (b) The loss is Lipschitz with the model scores so that a generalization bound associated with the loss holds, as in Lemma \ref{gen-bound}.
\end{proofsketch}

\paragraph{Reducing the generalization gap.}
The bound \eqref{err-bound} contains three parts: the first line is the worst risk that can be achieved by $f_\rho$ as we described in Proposition \ref{vio-risk-inequalities}, the second and the third line is the complexity of the classifiers that have a small violation, and the last line is the errors that are independent of the model.
This bound \eqref{err-bound} is most preferable when a large set of unlabeled data is available so that the approximation errors of violations (i.e., term $B(\delta/2, m_\:U, \mathcal{F})$, ${\mathfrak{R}}_{m_\:U}(\+F_{1/\rho + u + B(\delta/2, m_\:U, \mathcal{F})})$ and $\sqrt{\frac{\log(1/\delta)}{2m_\:U}}$) are all small. Then, the model complexity is mainly described by the term ${\mathfrak{R}}_{m_\:L}(\+F_{1/\rho + u})$, which is the Rademacher complexity of a proper subset of $\+F$.
In this sense, the regularization method reduces the generalization gap by reducing the model complexity of the scoring space.

\paragraph{Tradeoff in regularization.}
In situations where $m_\:U$ is large, the tradeoff parameter $\rho$ balances two quantities: a larger $\rho$ leads to a smaller scoring space $\+F_{1/\rho + u}$, but brings more bias depending on the suboptimality of $f_0$ in violation, measured by $V(f_0)-V(f_\infty)$.
The benefit of regularization is greater if fewer classifiers can achieve a violation that is close to the optimal value $V(f_\infty)$.

We provide the following example to illustrate how the Rademacher complexity can be reduced in linear models.

\begin{example}[Logistic Regression]
\label{rademacher-example}
Consider a linear model for multiclass classification where $\+Y=[c]$ and $f(x,j)=w_j^{\:T} x$ with $\sum_{j=1}^c \|w_j\|_2^2 \le 1$. 
Suppose $x \in \-R^p$ is distributed in the unit sphere $\|x\|_2 \le 1$ with expectation $\-E[x] = \alpha \in \-R^p$ and covariance matrix $\sigma^2I_{p\times p}$. 
Without constraint, the Rademacher complexity is upper bounded as $\mathfrak{R}_m(\+F) \le \sqrt{c/m}$ as in \citet{CKMY16} (Theorem 2).
Now, consider a constraint that removes exactly one label so that $C(x) \equiv [c-1]$.
With regularization, for sufficient small $t<1/(c+2)$, we have the following bound 
\begin{equation}
    \mathfrak{R}_m(\+F_t) 
    \le \frac{1}{2}\left(\sqrt{\frac{c}{m}} + \sqrt{\frac{c-\sigma^2-\|\alpha\|_2^2}{m}}\right)
\end{equation}
which is strictly tighter than the standard bound. Intuitively, if $x$ is concentrated around the origin $\mathbf{0}$, the prediction by any classifier will tend to be a uniform distribution. Therefore, a large bias and variance in $x$ (captured by $\sigma^2+\|\alpha\|_2^2$) help to distinguish models with different levels of violation.
\end{example}

\paragraph{Compare to existing results.}
Previous works mostly consider a zero-one loss for both classification and violation under the assumption that the risk minimizer also achieves zero violation.
Then, one can simply preclude all the classifiers $f\in \+F$ that have nonzero empirical violations on the unlabeled dataset and find the ERM among the remaining classifiers. 
This approach has been theoretically studied in \citet{BalcanBl05, BalcanBl10} for binary classification and \citet{TulabhuRu14} in a similar manner for regression by characterizing the complexity of the reduced set of hypotheses that achieve zero violation.
Conceptually, we can regard this algorithm as a special case of problem \eqref{regularized-objective} when $\rho = \infty$.
Our study, therefore, extends previous works with a soft learning objective to multiclass classification problems. 

%% file: sections/inference.tex
\section{Inference with Constraints}
\label{inference}

An \emph{inference algorithm} is a mapping $\+F \times \+X \rightarrow \Delta_{\+Y}$. 
By default, we define it as the softmax inference: $(f,x) \mapsto \-P_f(\cdot|x)$.
When performing \emph{inference with constraints} (or constrained inference), we modify this softmax mapping for the given function $f$ using the additional information of $C$.

In this section, we study the Constrained Conditional Model (CCM) \citep{CRRR08, ChangRaRo12}, a broad family of models that perform inference with constraints.
We show at testing time, whether CCM reduces the risk depends on whether the model's expected violation is larger than the noise rate of the constraint $V_\ora$ (Theorem \ref{change-of-risk}).
In particular, when the constraint is noise-free, CCM always achieves a smaller or equal risk.
Furthermore, we show better risks are achieved if the constrained inference is also performed at training time, and pursuing this optimal risk leads to a learning objective that contrasts with the one used in the regularization approach (Proposition \ref{on-training}).

To make distinguishment, we will refer to a model in the original spaces $\+F$ as a \emph{base model} and refer to an augmented model as a \emph{constrained model}.

\subsection{Constrained Conditional Model}

CCM augments existing scoring functions using a linear combination with the violation function. Precisely, given a vanilla scoring space $\+F$, the scoring space of CCM is defined as follows.

\begin{definition}[Constrained conditional model \citep{CRRR08, ChangRaRo12}] 
    Given a scoring space $\+F$, a constraint $C$ and a fixed tradeoff parameter $\mu \in [0, \infty]$, the scoring space of the Constrained Conditional Model (CCM) is defined as:
    \begin{equation}
    \begin{aligned}
        \+F^{\mu}
        := \left\{ (x,y) \mapsto f(x,y) - \mu v_{C}(x,y) \middle| f\in \+F\right\}
    \end{aligned}
    \end{equation}
    We will also denote 
    \begin{equation}
    \label{CCM}
    \begin{aligned}
        f^\mu(x,y) 
        := f(x,y) - \mu v_{C}(x,y)
    \end{aligned}
    \end{equation}
    to be the augmented scoring function for a given $f\in\+F$. In particular, setting $\mu = \infty$ will assign a score $-\infty$ to any $y \notin C(x)$, which implies $\-P_{f^\infty}(y|x)=0$, namely forcing strictly-constrained inference.
\end{definition}

The tradeoff parameter $\mu$ allows CCM to improve the base model $f$ despite noisy constraints, as we will discuss in detail in the following sections. Otherwise, if the noise rate is large, performing strictly-constrained inference can be harmful because it assigns 0 probability mass to any label $y$ that is outside $C(x)$ and hence has a classification loss $L(x,y_\ora,f^\infty)=1$ at any $x$ where $y_\ora \notin C(x)$.

The learner can choose whether or not to perform the constrained inference either at the training time. This choice leads to the following two different approaches:
\begin{enumerate}
    \item 
    On-training approach: perform constrained inference both at training and testing time, and directly find the ERM over $\+F^\mu$ using labeled data (also known as \emph{(Inference Based Training} in \citep{PRYZ05})
    
    \item 
    Post-training approach: first find the ERM over the vanilla $\+F$ using labeled data, and then perform constrained inference at the testing time (also known as \emph{Learning Plus Inference} in \citep{PRYZ05}). 
\end{enumerate}
For both approaches, the generalization ability of CCM is characterized by the complexity of $\+F^{\mu}$. So, we first point out that CCM does not increase the Rademacher complexity.

\begin{proposition}[Rademacher Complexity of CCM]
    \label{CCM-Rademacher}
        For any fixed $\mu \ge 0$ and $m \in \-N$, we have the following identity:
        \begin{equation}
        \begin{aligned}
            {\mathfrak{R}_m}(\+F^{\mu})
            = {\mathfrak{R}_m}(\+F)
        \end{aligned}
        \end{equation}
\end{proposition}

\subsection{Post-training Constrained Inference}

For a given and fixed classifier $f$ (presumably trained with data), how does performing constrained inference impact the model performance?
In this section, we study the change in risk when the learner chooses to augment $f$ as a CCM $f^\mu$ defined in \eqref{CCM}. 
It is most convenient to characterize the risk of a CCM using the cross-entropy loss, although we will also conduct the same analysis for the hinge and $\ell^1$ losses, as we will point out later. 
To start with, for any $f$ and $\mu \in [0, \infty]$, we let 
\begin{equation}
    \Delta^\mu_\ce(f)
    :=R_\ce(f) - R_\ce(f^\mu)
\end{equation}
be the difference in the risk between the base model and the CCM (the larger the better).

\begin{theorem}[Change in cross-entropy risk]
\label{change-of-risk}
    We have:
    \begin{enumerate}[label=(\alph*)]
    \itemsep-0.5em 
        \item 
        For any fixed model $f$, there exists an $\mu_0 > 0$ such that $R_\ce(f^{\mu_0}) < R_\ce(f)$ if and only if 
        \begin{equation}
        \begin{aligned}
            V(f) > V_\ora
        \end{aligned}
        \end{equation}

        \item 
        The change in risk can be lower bounded as
        \begin{equation}
        \label{risk-bound-ce-loss}
        \begin{aligned}
            \Delta^\mu_\ce(f)
            \ge V(f)(1-\e^{-\mu}) - \mu V_\ora
        \end{aligned}
        \end{equation}

        \item 
        In particular, if the constraint is noise-free, we have 
        \begin{equation}
        \begin{aligned}
            \Delta^\infty_\ce(f)
            = V_\ce(f)
        \end{aligned}
        \end{equation}

    \end{enumerate}
\end{theorem}

The first result describes the sufficient and necessary condition for constrained inference to be helpful. 
It requires $f$ to have a larger violation (measured by $\ell^1$ violation) than the true data on average so that it has the potential to be improved. This condition is easier to satisfy when the constraint is less noisy.

The second result further quantifies the risk reduction as an explicit function of $\mu$.
The last result shows that in the noise-free case, the maximum risk reduction is exactly the expected violation measured by cross-entropy. Its consequences will be further discussed in the next section.

\begin{remark}[CCM with alternative losses]
    We present the counterparts of Theorem \ref{change-of-risk} for hinge loss and $\ell^1$ loss in the Appendix \ref{alternative-loss}. 
    The information delivered by those results is consistent with Theorem \ref{change-of-risk} in the sense that (1) whether CCM can reduce the risk depends on the comparison between the violation of the original model and the oracle. 
    (2) the reduction can be described or lower bounded by some measures of the violation.
    
    The drawback of the hinge loss is its non-smoothness due to the discontinuity of the argmax inference. The drawback of the $\ell^1$ loss is that the range of $\mu$ such that $R(f^\mu) \le R(f)$ can be disconnected and difficult to describe. Therefore, we provide weaker results by deriving only sufficient or necessary conditions for CCM to reduce the risks.
\end{remark}

As an application of Theorem \ref{change-of-risk}, we derive a sufficient condition under which CCM achieves smaller risks.
\begin{corollary}[Choice of $\mu$]
    \label{mu-upper-bound}
    Assuming $V(f) \ge V_\ora$, then $R_\ce(f^\mu) \le R_\ce(f)$ if the following condition holds:
    \begin{equation}
    \label{CCM-choice-of-mu}
    \begin{aligned}
        \mu 
        \le W(-\eta/\e^\eta)+\eta
    \end{aligned}
    \end{equation}
    where $\eta := V(f)/V_\ora$ is the relative violation rate and $W$ is the Lambert $W$ function whose value $W(t)$ is defined to be the solution to the equation $w\e^w = t$ of $w$.
\end{corollary}

The RHS of \eqref{CCM-choice-of-mu} increases with $\eta$ and vanishes as $\eta \rightarrow 1$.
In particular, when the constraint is noise-free, one should encourage strictly-constrained inference and set $\mu = \infty$. We also provide a plot of the RHS in the proof in the appendix.

\subsection{On-training Constrained Inference}
\label{section:on-training}

In this subsection, we study the on-training approach where we perform constrained inference both at the training and testing time. We use the results we established in the last subsection to describe the learning objective of the on-training approach, and argue that it achieves better risks than the post-training approach. Based on this, we further show that minimizing the cross entropy over CCM encourages a large violation of the base model, which contrasts the learning objective \eqref{regularized-objective-ce} that is used in regularization.

We provide a simplified analysis for the noise-free setting where we choose $\mu = \infty$ and perform strictly-constrained inference.
Then, the on-training approach aims to find the optimal (in terms of cross entropy) base model as follows:
\begin{equation}
\label{on-training-def}
\begin{aligned}
    \on :=
    \argmin_{f \in \+F} R_\ce(f^\infty)
\end{aligned}
\end{equation}
(recall $f^\infty$ means performing strictly-constrained inference with $f$) We characterize the behavior of $\on$ with the following results, which are direct corollaries of Theorem \ref{change-of-risk}.

\begin{proposition}[Learning Optimal CCM; Post-training vs On-training]
\label{on-training}
Assuming $C$ is noise-free, we can reformulate the learning objective \eqref{on-training-def} as
    \begin{equation}
    \label{on-training-objective-ce}
        \begin{aligned}
            \on
            = \argmin_ {f\in\+F} R_\ce(f) - V_\ce(f)
        \end{aligned}
    \end{equation}
\end{proposition}

\paragraph{A fundamental difference.}
Surprisingly, the reformulated learning objective \eqref{on-training-objective-ce} is opposite to the surrogate regularized objective defined in \eqref{regularized-objective-ce} in their attitudes towards violations. This contrast suggests a fundamental difference between regularization and constrained inference: the regularization method views violation as a bad thing and it precludes classifiers with substantial violations. But constrained inference corrects a model from its violation, so a large violation means a great potential to be improved.

\paragraph{On-training vs post-training.}
Loosely speaking, this result also suggests that in general, the best constrained model is not the constrained best model. To be more precise, suppose we perform post-training constrained inference for the cross-entropy risk minimizer in the vanilla model, i.e., $\post := \argmin_{f\in\+F} R_\ce (f)$.
Then, we can reformulate the definition of $\post$ as
\begin{equation}
\begin{aligned}
    \post 
    := \argmin_{f\in\+F} \underbrace{(R_\ce(f) - V_\ce(f))}_{\text{objective in \eqref{on-training-objective-ce}, post-training risk}} + V_\ce(f)
\end{aligned}
\end{equation}
which can be regarded as a ``regularized'' version of \eqref{on-training-objective-ce}. Therefore, similar to Proposition \ref{vio-risk-inequalities}, we can argue that the risk minimizer $\post$ over $\+F$, as a base model of CCM, contains a bias towards a higher risk than the on-training method's as follows:
\begin{equation}
\label{post-vs-on}
\begin{aligned}
    R_\ce(\on^\infty) 
    \le R_\ce(\post^\infty)
    \le R_\ce(\on) - \min_{f\in\+F} V_\ce(f)
\end{aligned}
\end{equation}
The proof is included in the proof of Proposition \ref{on-training}.

\paragraph{Computational considerations.} 
In practical structured prediction problems where the output is sequential or graphical, performing constrained inference during training time is typically expensive due to the complexity of the constraints. For example, as pointed out by \citet{pmlr-v80-xu18h}, when the constraint is defined by a logical expression over several output variables, computing the probability of constraint being satisfied corresponds to the problem of weighted model counting (WMC)  and is \#P-complete \citep{Roth96}.
Therefore, to implement the on-training approach in practice, one can alternatively use approximate inference to ensure tractability.
For example, strictly constrained inference, formulated as Integer Linear Programming \citep{RothYi04}, can be further relaxed as Linear Programming \citep{martins-etal-2009-concise}.
Another example is \emph{amortized inference} \citep{CUKR15}, which accelerates the convergence to the optimal model while only performing exact inference in every $\tau>1$ iterations. 

\paragraph{Compare to existing results.}
There has been limited theoretical work discussing the impact of performing constrained inference. The most related one is \citet{PRYZ05}, which derives VC-style generalization bounds for linear structured models to argue that (1) performing strictly constrained inference in a post-training manner (\emph{Learning Plus Inference} in the paper) improves the model performance and (2) the on-training approach (\emph{Inference Based Training} in the paper) further reduces the error in the long run. Our approach directly analyses the classification risk and extends the comparison to noisy constraints and soft-constrained inference with CCM.

%% file: sections/combine.tex
\section{Regularization with Constrained Inference}
\label{combo}

\begin{figure}[h]
    \centering
    \includegraphics[width=0.4\textwidth]{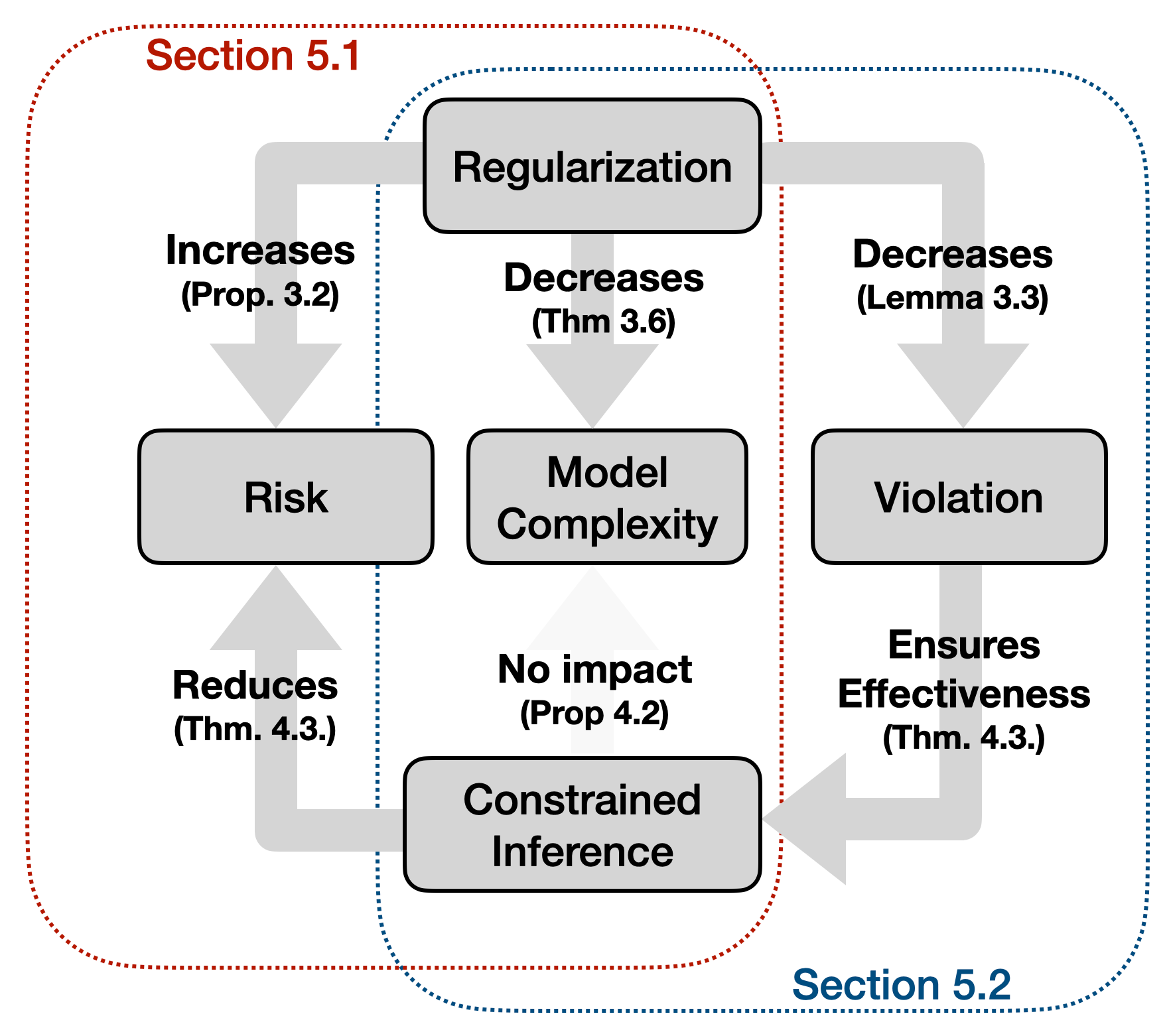}
    \caption{A summary of the established results, as motivations to the problems of this section. In Section \ref{combo-help}, we describe how CCM can reduce the additional risk introduced by regularization. In section \ref{combo-not-help}, we claim that the decrease of violation with regularization can make post-training constrained inference unnecessary.}
    \label{combo-visual}
\end{figure}

We have seen that regularization and constrained inference have different impacts on the generalization gap and the risk.
On one hand, CCM has an equal Rademacher complexity (Proposition \ref{CCM-Rademacher}) as the original model $\mathfrak{R}(\+F)$, which can be reduced by regularization. So, performing regularized algorithm to CCM also reduces the generalization gap.
On the other hand, their impacts on the risks are contradicting, as summarized in figure \ref{combo-visual}. 
In this section, we aim to describe how these impacts can interact with each other by applying our established results to explore the usage of these two methods together.

We show both positive and negative results for the combination. On one hand, we propose sufficient conditions under which the bias introduced by regularization can be compensated by performing constrained inference (Proposition \ref{combo-help}). 
On the other hand, we study if post-training constrained inference can reduce the risk of the optimal classifier $f_\rho$. We show with a noisy constraint, choosing a large value of $\rho$ in the regularized objective \eqref{regularized-objective} will make CCM incapable to reduce the risk (Proposition \ref{combo-not-help}).

\subsection{CCM Compensates for Regularization Bias} 

As the red part of Figure \ref{combo-visual} summarizes, we have shown that the regularization and constrained inference have contradicting influences on the risk. Moreover, the regularization bias is controlled by the violation of the risk minimizer (Proposition \ref{vio-risk-inequalities}), which can be reduced by constrained inference. This suggests the possibility for CCM to reduce the additional risk introduced by regularization.

We formally describe this phenomenon by considering the following combination: an on-training approach that aims to find the minimizer of the following regularized surrogate objective over the CCM $\+F^\mu$:
\begin{equation}
    \label{combo-objective}
    \begin{aligned}
        f_\star^\mu
        := \argmin_{g\in\+F^\mu} R_\ce(g) + \rho V_\ce(g)
    \end{aligned}
\end{equation}
Recall that $R_\ce(\post)$ is the minimum cross-entropy risk that can be achieved in $\+F$.
We show that unlike the vanilla regularized objective \eqref{regularized-objective}, it is possible for this algorithm to achieve a smaller risk than $R_\ce(\post)$ as follows. 

\begin{theorem}[Regularization with on-training constrained inference]
    \label{combo-help}
    If 
    CCM improves $\post$ so that $\Delta^\mu_\ce(\post)> 0$, 
    then letting
    \begin{equation}
        \label{sufficient-combo-help}
        \begin{aligned}
            \rho 
            < \frac{V_\ce(\post)-\mu V_\ora}{V_\ce(\post^\mu)} - 1
        \end{aligned}
    \end{equation}
    will imply $R_\ce(f_\star^\mu) < R_\ce(\post)$.
\end{theorem}

This result shows a small choice of $\rho$ allows the regularized optimizer $f_\star^\mu$ to achieve better cross-entropy.
A less noisy constraint allows more choices of $\rho$ to make this happen.
In particular, when the constraint is noise-free, since $V_\ce(\post^\mu) \rightarrow 0$ as $\mu \rightarrow \infty$, driving $\mu$ to $\infty$ will make $R(f_\star^\mu) < R(\post)$ for all $\rho > 0$. 
As a cost, regularization will be less effective in reducing the Rademacher complexity with a large value of $\mu$. In the extreme case, all the classifiers in $\+F^\infty$ make zero violation, and hence cannot be distinguished by the regularization objective.

\subsection{Post-regularized-training Constrained Inference}

Finally, as the blue part of Figure \ref{combo-visual} summarizes, we have shown that post-training inference is beneficial only if the average violation of $f$ is larger than $V_\ora$ (Theorem \ref{change-of-risk}). However, the minimizer of the regularized objective $f_\rho$ tends to have a small violation (Proposition \ref{f-rho-violation}) scaled with $1/\rho$.
Therefore, it is possible that choosing a large value of $\rho$ will make post-training incapable to reduce the risk with a noisy constraint.
Formally, assuming a model is already trained with the vanilla regularized $\ell^1$ objective as in \eqref{regularized-objective}, we have the following holds.

\begin{theorem}[When post-training inference is not helpful for regularized model]
    \label{combo-not-help}
    Recall $V(f_\infty)$ is the minimal expected violation that can be achieved by $\+F$. If $V_\ora \ge V(f_\infty)$ and
    \begin{equation}
    \label{combo-not-help-condition}
    \begin{aligned}
        \rho 
        \ge \frac{1}{V_\ora - V(f_\infty)}
    \end{aligned}
    \end{equation}
   then the minimizer $f_\rho$ of the regularized objective \eqref{regularized-objective} will \emph{not} be improved by post-training constrained inference for any $\mu \in (0, \infty]$ in the sense that $R_\ce(f_\rho) \le R_\ce((f_\rho)^\mu)$.
\end{theorem}

The RHS of \eqref{combo-not-help-condition} shrinks with a larger noise rate $V_\ora$ and smaller $V(f_\infty)$. Intuitively, a more noisy constraint is less helpful (Theorem \ref{change-of-risk}), while a small value of $V(f_\infty)$ allows $f_\rho$ to violate less (Proposition \ref{f-rho-violation}) and hence gains fewer benefits from constrained inference (Theorem \ref{change-of-risk}).
As a consequence, with a noisy constraint, choosing a large $\rho$ in the regularized objective will make post-training constrained inference unnecessary or even harmful.

%% file: sections/related-work.tex
\section{Related Works}

\paragraph{Regularization with constraints.}
In the context of structured prediction, the Posterior Regularization (PR) framework \citep{GGGT10} proposed to regularize the log-likelihood by adding a distance of the probabilistic prediction to the constrained subspace of distributions.
The CoDL algorithm \citep{ChangRaRo07, ChangRaRo12} is a semi-supervised algorithm that repetitively assigns constrained pseudo-labels to the unlabeled dataset and uses pseudo-labels to retrain the model. 
CoDL and PR are further unified in \citet{SamdaniChRo12} as special cases of a parameterized EM algorithm.
More recent works have proposed injecting logical constraints into deep models by augmenting the training objective with explicitly defined violation functions, such as the semantic loss \citep{pmlr-v80-xu18h}, the DL2 loss \citep{FBDGZV19} and the inconsistency loss \citep{Li2019ALF}, which motivate our theoretical formulation in \eqref{regularized-objective}.

\paragraph{Inference with constraints.}
The idea of injecting prior knowledge directly into a predictive model dates back to \citet{RothYi04}, which formulates the problem of inference with hard constraints as Integer Linear Programming (ILP).
The idea of constrained inference has been followed and developed by NLP researchers and empirically shown to be effective in various problems such as summarization \citep{ClarkeLa08}, temporal reasoning \citep{NingWuRo18}, semantic parsing \citep{scholak-etal-2021-picard} and text generation \citep{lu-etal-2022-neurologic}.
\citet{CRRR08, ChangRaRo12} further defines the CCM to incorporate soft constraints into linear models.
Another related work is \citet{ENRIQUESUCAR201414}, which uses Bayesian networks to model the label correlations and define an order to the labels.
The order information is then taken as extended features at inference time.
Theoretically, \citet{PRYZ05} provides a comparison between the on-training and post-training constrained inference using VC-style error bounds.

\paragraph{Semi-supervised learning theory.} 
Several theoretical semi-supervised learning frameworks such as \citet{BalcanBl05, BalcanBl10} and \citet{TulabhuRu14} illustrate how hard constraints on the hypothesis space could reduce the generalization error. A detailed comparison can be seen in the discussion at the end of Section \ref{learning}.

\paragraph{Learning with partial labels.}
The problem of learning with constraints is closely related to the problem of learning from partial labels (also known as superset labels) \citep{cour2011,learnability-pll, proper-losses-pll, structured-prediction-pll} where each instance $x$ in the dataset is assigned with a partial label $s$ which also takes value in $ 2^\+Y$. 
The difference is that the constraint mapping itself is known to the learner and hence can be encoded in the inference algorithm directly, for example, via the CCM. Another difference is that the partial labels are typically more informative and can guarantee learnability alone \citep{learnability-pll, WangNiRo20}. In contrast, the constraints that appear in practice typically provide only side information and need to be used with gold labels together.

%% file: sections/conclusion.tex
\section{Conclusion and Future Works}

In this paper, we presented a theoretical study of two methods to encode label constraints into a learning system: regularization and constrained inference. 
We compared these two approaches by quantifying their impact on the optimal risk as well as the generalization error. 
Our study revealed that the success of these two approaches replies on different data assumptions:
the regularization method requires the optimal classifier in the model to have a small violation while constrained inference requires the true data to have a small violation.
We further elucidated the detrimental consequences that arise when these assumptions fail to hold.
Finally, we demonstrate how their impacts on the model can interact when used together.

We have focused on multiclass classification, aiming to provide a starting point for understanding the different mechanisms of the two methods. For future work, we will extend the discussion to structured prediction problems where complex constraints are naturally defined. In particular, while the presence of constraints can improve the model performance, it also suggests a strong dependency inside the structure, which may hurt the generalization performance, as pointed out by \citet{LondonHuGe16}.

%% file: sections/appendix.tex
\appendix 
\onecolumn

\part*{Appendix}

\section{Details on Loss Function}
\label{loss}

The $\ell^1$ loss is a smoothed alternative to the zero-one loss and has been used in the theoretical analysis for the generalization error, see, for example, in \citet{LondonHuGe16} (Section 6.2). It can be related to other common loss functions as follows.

\paragraph{As distances on the probability simplex.}
Let $e_y \in \-R^{|\+Y|}$ be a one-hot vector with the $y^\mathrm{th}$ coordinate be $1$ and all others be $0$. We then have that
$$
    L(x,y_\ora,f) 
    := 1 - \-P_f(y_\ora|x) 
    = \frac{1}{2}\|e_{y_\ora} - \-P_f\|_1
$$
Moreover, since our label space $\+Y$ is of finite cardinality, we further have that $\frac{1}{2}\|e_{y_\ora} - \-P_f\|_1 = \operatorname{TV}(e_{y_\ora}, \-P_f)$, the total variation distance.

\paragraph{Relation to zero-one loss.}
By introducing a temperature parameter $t \in \-R_{\ge 0}$ to the softmax function, it is well known that $\lim_{t \rightarrow \infty} \softmax(tu) = \argmax(u)$ for a vector $u$. This implies 
$$
    \lim_{t \rightarrow \infty} L(x,y_\ora,tf) 
    = 1 - \-1\{\argmax_{y\in\+Y}f(x,y) = y_\ora\}
    = \-1\{\argmax_{y\in\+Y}f(x,y) \ne y_\ora\}
$$
which is the zero-one loss.
Since performing softmax inference with temperature $t$ can be equivalently regarded as performing softmax inference for the scoring space $t\+F$, for the simplicity of our presentation, we omit the temperature parameter in the softmax inference.

\paragraph{Relation to cross-entropy.}
The total variation distance to a one-hot probability can be lower bounded by cross-entropy due to Pinsker's inequality. More directly, in our case, we have $1-p \le -\log(p)$ for any $p \in [0,1]$ from basic inequality. This implies $L(x,y,f) \le L_\ce(x,y,f)$.

In conclusion, the $\ell^1$ loss is a $\ell^1$ and total variation distance on the probability space, is a smoothed version of the zero-one loss, and is upper bounded by cross-entropy. It is differentiable and bounded so that we can derive generalization bounds with Rademacher complexity. Another reason that we are interested in softmax inference will be clearer in the discussion for constrained inference, where in Theorem \ref{change-of-risk}, \ref{change-of-risk-hinge} and \ref{change-of-risk-l1}, the change of expected cross entropy and $\ell^1$ loss can be lower bounded by a smooth function. But with the argmax inference, the risk is in general not continuous and needs to be assumed to be Lipschitz to obtain similar results.

\section{Proofs from Section 3}

\subsection{Proof of Proposition \ref{vio-risk-inequalities}}

The first inequality is straightforward. For the second inequality, by definition \eqref{regularized-objective} we have 
\[
\begin{aligned}
    R(f_\rho) + \rho V(f_\rho) 
    \le R(f_0) + \rho V(f_0)
\end{aligned}
\]
and 
\[
    V(f_\rho) \ge V(f_\infty)
    .    
\]
Combining the two above inequalities yields
\[
\begin{aligned}
    R(f_\rho) + \rho V(f_\infty)
    \le R(f_0) + \rho V(f_0)
    .
\end{aligned}
\]
The desired inequality follows by rearranging these terms. This argument also holds if we replace the expectations with empirical estimates.

To see how the RHS bound can be reached, consider the following scoring space that contains two classifiers, $f_0$ and $f_\infty$, and an instance space $\+X$ that only contains one point $x$. Let $C(x) = \{y_\ora,y'\}$. Let $f_0$ be such that $\-P_{f_0}(y_\ora)=a\in(0,1)$ and $\-P_{f_0}(y')=b$. Let $f_\infty$ be such that $\-P_{f_\infty}(y_\ora)=a-\epsilon_1$ and $\-P_{f_\infty}(y')=b+\epsilon_2$ so that $\epsilon_1 < \rho \epsilon_2$. Then 
\begin{equation}
    R(f_\infty) + \rho V(f_\infty) 
    \le 1 - (a - \epsilon_1) + \rho (b-\epsilon_2)
    < 1-a + \rho b
    = R(f_0) + \rho V(f_0) 
\end{equation}
which means $f_\infty$ will be preferred to $f_0$ by the regularized objective.

\subsection{Proof of Lemma \ref{f-rho-violation}}
By definitions, we have
\begin{equation}
    \begin{aligned}
        \rho {V}({f}_\rho) 
        & \le {R}({f}_\rho) +  \rho {V}({f}_\rho) \\ 
        & \le {R}(f_\infty) +  \rho {V}(f_\infty) \\ 
        & \le 1 +  \rho {V}(f_\infty) \\ 
        & \le 1 + \rho u
    \end{aligned}
\end{equation}
Therefore, we have that $V(f_\rho) \le u + 1/\rho$.

\subsection{Proof of Lemma \ref{gen-bound}}

To prove this theorem, we need the following lemmas. The first one is a contraction inequality established in \citep{CKMY16}.

\begin{lemma}[Lemma 5 from \citep{CKMY16}]
\label{contraction}
    Let $\+H$ be a set of functions mapping $\+X$ to $\-R^N$. Suppose $\Phi_i$ is $\mu_i$-Lipschtz with the 2-norm, i.e.,
    \begin{equation}
    \begin{aligned}
        |\Phi_i(v') - \Phi_i(v)| 
        \le \mu_i \|v'-v\|_2 
        \quad \forall v,v'\in\-R^N
    \end{aligned}
    \end{equation}
    Then for any set of $m$ points $x_1,\dots, x_m \in \+X$, the following inequality holds
    \begin{equation}
    \begin{aligned}
        \frac{1}{m} \-E_\sigma \left[
            \sup_{h \in \+H} \sum_{i=1}^m \sigma_i \Phi_i(h(x_i)) 
            \right]
            \le \frac{\sqrt{2}}{m} \-E_\epsilon \left[
                \sup_{h\in\+H} \sum_{i=1}^m \sum_{j=1}^N \epsilon_{ij} \mu_i h_j(x_i)
        \right]
    \end{aligned}
    \end{equation}
    where $\sigma_i$s and $\epsilon_{ij}$s are independent Rademacher variables uniformly distributed over $\{-1,+1\}$.
\end{lemma}

The second one computes the Lipschitz constants of the $\ell^1$ losses by bounding its gradient's 2-norm.

\begin{lemma}[Lipschitzness]
    Given a scoring function $f:\+X \times \+Y \rightarrow \-R$, let $f(x) = [f(x,y)]_{y \in \+Y} \in \-R^{|\+Y|}$ be the vector of scores for each label.
    For any two scoring functions $f,f'$ and data $(x,y)$, we have that
    \begin{equation}
    \begin{aligned}
        |\-P_f(y|x) - \-P_{f'}(y|x)|
        \le \frac{\sqrt{2}}{4} \|f(x) - f'(x)\|_2
    \end{aligned}
    \end{equation}
    Furthermore, for any constraint $C$, we have 
    \begin{equation}
        \begin{aligned}
            |\-P_f(C|x) - \-P_{f'}(C|x)|
            \le \frac{1}{4}\sqrt{1 + \frac{1}{|C(x)|}} \|f(x) - f'(x)\|_2
        \end{aligned}
    \end{equation}
    where $\-P_f(C|x)=\-P_f(C(x)|x)=\sum_{y \in C(x)} \-P_f(y|x)$.
\end{lemma}

\begin{proof}
We start with the second claim. 
Suppose $C(x) = \+Y$, then $\-P_f(C|x) = 0$ for any scoring function $f$, so the inequality trivially holds. 
Next, we assume $C(x) \subset \+Y$.
Given a constraint $C:\+X \rightarrow 2^\mathcal{Y}$, the derivative of its violation function with respect to the score for a label $y$ is
\begin{equation}
    \begin{aligned}
        \frac{\dd \-P_f(C|x)}{\dd f(x,y)}
        & = \sum_{y' \in C(x)} \frac{\dd \-P_f(y'|x)}{\dd f(x,y)} \\
        & = \sum_{y' \in C(x)}  \-P_f(y|x) \-1\{y' = y\} -  \-P_f(y|x) \-P_f(y'|x)
    \end{aligned}
\end{equation}

The 2-norm of the gradient of the mapping $f(x) \mapsto \-P_f(y|x)$ is then
\begin{equation}
\begin{aligned}
    \left(
        \sum_{y \in \+Y} \left( \sum_{y' \in C(x)}  \-P_f(y|x) \-1\{y' = y\} -  \-P_f(y|x) \-P_f(y'|x) \right)^2
    \right)^{1/2}
\end{aligned}
\end{equation}
which is maximized when $\-P_f(y|x) = \frac{1}{2|C(x)|}$ for all $y \in C(x)$ and $\-P_f(y|x) = \frac{1}{2(\+Y-|C(x)|)}$ for all $y \notin C(x)$ (so that $\-P_f(C|x)=1/2$). The maximum is then 
\begin{equation}
\label{eqn:lip}
\begin{aligned}
    & \left(
         \sum_{y \in C(x)} \left( \sum_{y' \in C(x)}  \-P_f(y|x) \-1\{y' = y\} -  \-P_f(y|x) \-P_f(y'|x) \right)^2
        + \sum_{y \notin C(x)} \left( \sum_{y' \in C(x)}  \-P_f(y|x) \-P_f(y'|x) \right)^2
    \right)^{1/2} \\ 
    & = \sqrt{|C(x)|\left(\frac{1}{4|C(x)|}\right)^2 + |\+Y-C(x)| \left(\frac{1}{2|\+Y-C(x)|}\right)^2} \\ 
    & = \sqrt{\frac{1}{16 |C(x)|} + \frac{1}{16|\+Y-C(x)|}} \\
    & \le \sqrt{\frac{1}{16 |C(x)|} + \frac{1}{16}} \\ 
    & = \frac{1}{4}\sqrt{1 + \frac{1}{|C(x)|}}
\end{aligned}
\end{equation}
The boundedness of the gradient implies that the function $f(x) \mapsto \-P_f(C|x)$ is Lipschitz with a Lipschitz constant $\frac{1}{4}\sqrt{1 + \frac{1}{|C(x)|}}$.

The first claim then follows by considering the special constraint $C(x) := \{y_\ora(x)\}$ so that $|C(x)| = 1$.
\end{proof}

Next, we present the proof of the theorem. By standard Rademacher complexity bounds, given a labeled dataset $S$ of size $m$, for any $\delta>0$, with probability at least $1-\delta$, the following inequality holds uniformly for $f \in \+F$:
\begin{equation}
\begin{aligned}
    R(f) 
    \le \hat{R}(f;S_\:L) + 2 \mathfrak{R}_m(\+H) + \sqrt{\frac{\log (1/\delta)}{2m}}
\end{aligned}
\end{equation}
where
\begin{equation}
    \begin{aligned}
        \+H
        := \left\{(
            x,y) \mapsto 1- \-P_f(y|x): f \in \+F
        \right\}
    \end{aligned}
\end{equation}

By the contraction lemma and Lipschitzness, we have 
\begin{equation}
\begin{aligned}
    \mathfrak{R}_m(\+H) 
    & = \frac{1}{m} \-E_{S} \-E_\sigma \left[
        \sup_{f \in \+F} \sum_{i=1}^m \sigma_i \left( 1 - \-P_f(y_i|x_i)\right)
    \right] \\ 
    & \le \frac{\sqrt{2}}{m}  \-E_{S} \-E_\epsilon \left[
        \sup_{f \in \+F} \sum_{i=1}^m \sum_{y \in \+Y}\epsilon_{iy} \frac{\sqrt{2}}{4} f(x, y) 
    \right]\\ 
    & = \frac{1}{2m}  \-E_{S} \-E_\epsilon \left[
        \sup_{f \in \+F} \sum_{i=1}^m \sum_{y \in \+Y} \epsilon_{iy} f(x, y) 
    \right]
\end{aligned}
\end{equation}
This implies 
\begin{equation}
    \begin{aligned}
        R(f) 
        \le \hat{R}(f;S_\:L) + \mathfrak{R}_m(\+F) + \sqrt{\frac{\log (1/\delta)}{2m}}
    \end{aligned}
\end{equation}

The proof for the generalization bound of violation follows from the same argument. In particular, if the size of the constrained set $C(x)$ is a constant, namely $|C(x)|=c_0 < c = |\+Y|$ for all $x \in \+X$, then from Equation \eqref{eqn:lip}, we know that the mapping $x \mapsto 1- \-P_f(y|x)$ is Lipschitz with a Lipschitz constant $\frac{1}{4}\sqrt{\frac{1}{c_0} + \frac{1}{c-c_0}}$. So in this case, the generalization bound for the violation function can be improved as 
\begin{equation}
\begin{aligned}
    V(f) 
            \le \hat{V}(f;S_\:U) 
            + \frac{\sqrt{2}}{2}\sqrt{\frac{1}{c_0} + \frac{1}{c-c_0}}\mathfrak{R}_{m_\:U}(\+F) 
            + \sqrt{\frac{\log(1/\delta)}{2m_\:U}}
\end{aligned}
\end{equation}

\subsection{Proof of Theorem \ref{bound-regularization}}

\ulit{Step 1. Showing the expected violation of $\hat{f_\rho}$ is bounded.}

First, we have with probability $1-\delta$,
\begin{equation}
\begin{aligned}
    \rho \hat{V}(\hat{f}_\rho) 
    & \le \hat{R}(\hat{f}_\rho) +  \rho \hat{V}(\hat{f}_\rho) \\ 
    & \le \hat{R}(f_\infty) +  \rho \hat{V}(f_\infty) \\ 
    & \le 1 +  \rho \hat{V}(f_\infty) \\ 
    & \le 1 + \rho\left(u + \sqrt{\frac{\log(1/\delta)}{2m_\:U}} \right)\\ 
\end{aligned}
\end{equation}
where the last step follows by applying Hoeffding's inequality to $\hat{V}(f_\infty)$. This result implies $\hat{V}(\hat{f}_\rho) \le \frac{1}{\rho} + u + \sqrt{\frac{\log(1/\delta)}{2m_\:U}}$.

Second, Theorem \ref{gen-bound} claims that with probability $1-\delta$, the following inequality holds:
\begin{equation}
\begin{aligned}
    V(\hat{f}_\rho) - \hat{V}(\hat{f}_\rho) \
    \le \mathfrak{R}_{m_\:U}(\+F) + \sqrt{\frac{\log(1/\delta)}{2m_\:U}}
\end{aligned}
\end{equation}

Putting these two inequalities together using union bound, we know with probability $1-2\delta$,

\begin{equation}
\begin{aligned}
    V(\hat{f}_\rho) 
    & \le \frac{1}{\rho} + u + \mathfrak{R}_{m_\:U}(\+F) + \sqrt{\frac{\log(1/\delta)}{2m_\:U}} + \sqrt{\frac{\log(1/\delta)}{2m_\:U}} \\ 
    & = \frac{1}{\rho} + u + B(\delta,m_\:U,\+F)
\end{aligned}
\end{equation}

Namely, with probability no less than $1-2\delta$, $\hat{f}_\rho$ lies in $\+F_{1/\rho + u + B(\delta,m_\:U,\+F)}$, which is a fixed hypothesis class.

\ulit{Step 2. Bounding the generalization gap of $\+L_\rho$.}

Since $\hat{f}_\rho \in \+F_{1/\rho + u + B(\delta,m_\:U,\+F)}$, we can bound the generalization gap of $\+L_\rho$ using the uniform convergence property of $\+F_{1/\rho + u + B(\delta,m_\:U,\+F)}$. By standard decomposition, 
\begin{equation}
\begin{aligned}
    \+L_\rho (\hat{f}_\rho) - \+L_\rho (f_\rho)
    & = 
        \underbrace{\+L_\rho (\hat{f}_\rho) - \hat{\+L}_\rho (\hat{f}_\rho)}_{(*)}
        + \underbrace{\hat{\+L}_\rho (\hat{f}_\rho) - \hat{\+L}_\rho (f_\rho)}_{\le 0}
        + \underbrace{\hat{\+L}_\rho (f_\rho) - \+L_\rho (f_\rho)}_{(**)}
\end{aligned}
\end{equation}
For term $(*)$, combining the two inequalities in Lemma \ref{gen-bound} and Step 1 via union bound, we know with probability $1-4\delta$,
\begin{equation}
\begin{aligned}
    (*) 
    \le \mathfrak{R}_{m_\:L}(\+F_{1/\rho + u + B(\delta,m_\:U,\+F)}) + \sqrt{\frac{\log(1/\delta)}{2m_\:L}} + \rho \left( \mathfrak{R}_{m_\:U}(\+F_{1/\rho + u + B(\delta,m_\:U,\+F)}) + \sqrt{\frac{\log(1/\delta)}{2m_\:U}}\right)
\end{aligned}
\end{equation}
For term $(**)$, using Hoeffding's inequality for the risk and violation separately, we have with probability $1-2\delta$,
\begin{equation}
\begin{aligned}
    (**) 
    \le \sqrt{\frac{\log(2/\delta)}{2m_\:L}} + \rho \sqrt{\frac{\log(2/\delta)}{2m_\:U}}
\end{aligned}
\end{equation}
By union bound, with probability $1-6\delta$,
\begin{equation}
\begin{aligned}
    \+L_\rho (\hat{f}_\rho) - \+L_\rho (f_\rho) 
    & \le \underbrace{\mathfrak{R}_{m_\:L}(\+F_{1/\rho + u + B(\delta,m_\:U,\+F)}) + \rho \mathfrak{R}_{m_\:U}(\+F_{1/\rho + u + B(\delta,m_\:U,\+F)}) + 2 \sqrt{\frac{\log(2/\delta)}{2m_\:L}} + 2\rho \sqrt{\frac{\log(2/\delta)}{2m_\:U}}}_{\text{for convenience, denote these terms as }B'}
\end{aligned}
\end{equation}

\ulit{Step 3. Bounding the risk of $f_\rho$.}

By Step 2, we have with probability $1-6\delta$,
\begin{equation}
\begin{aligned}
    R(\hat{f}_\rho) 
    & \le R(f_\rho) + \rho V(f_\rho) - \rho V(\hat{f}_\rho) + B' \\ 
    & \le R(f_0) + \rho V(f_0) - \rho V(\hat{f}_\rho) + B' \\
    & \le R(f_0) + \rho V(f_0) - \rho V(f_\infty) + B' 
\end{aligned}
\end{equation}
We conclude that with probability $1-6\delta$,
\begin{equation}
\begin{aligned}
    R(\hat{f}_\rho) 
    & \le R(f_0) + \rho V(f_0) - \rho V(f_\infty) \\ 
    & \quad + \mathfrak{R}_{m_\:L}(\+F_{1/\rho + u + B(\delta,m_\:U,\+F)}) + \rho \mathfrak{R}_{m_\:U}(\+F_{1/\rho + u + B(\delta,m_\:U,\+F)}) + 2 \sqrt{\frac{\log(2/\delta)}{2m_\:L}} + 2\rho \sqrt{\frac{\log(2/\delta)}{2m_\:U}}
\end{aligned}
\end{equation}
as claimed.

\subsection{Proof of Example \ref{rademacher-example}}

The normalizing factor $\sum_{j=1}^c \e^{w_j^\:T x}$ is maximized at $w_1=x=[1,0,0,\dots,0]$ and $w_2=\dots=w_c=0$ so that
\begin{equation}
\begin{aligned}
    \sum_{j=1}^c \e^{w_j^\:T x}
    \le \e + (c-1) 
    \le c+2
\end{aligned}
\end{equation}
This implies $\-P_w(y_c) \ge (\e^{w_c^\:T x})/(c+2)$. Therefore, $\-E[\-P_w(y_c)] \le t$ implies $t(c+2) \ge \-E[\e^{w_c^\:T x}] \ge \e^{\-E[w_c^\:T x]} = \e^{\alpha^\:T w_c}$, or equivalently $\alpha^\:T w_c \le \log(t(c+2))$. 
Therefore, given a set of data $S=\{x_i\}_{i=1}^m$ and Rademacher random variables $\epsilon$, the inner supremum in the definition of Rademacher complexity can be upper bounded by solving the following program
\begin{equation}
\begin{aligned}
    \max & \quad \sum_{i=1}^m \sum_{j=1}^c \epsilon_{i, j} w_j^\:T {x}_i \\ 
            \mathrm{s.t.} & \quad \sum_{j=1}^c w_j^\:T w_j \le 1 \\ 
            & \quad \alpha^\:T w_c \le \log(t(c+2))
\end{aligned}
\end{equation}
Consider its Lagrangian
\begin{equation}
\begin{aligned}
    L(w, \lambda, \mu) 
    = \sum_{i=1}^m \sum_{j=1}^c \epsilon_{i,j} w_j^\:T x_i 
        + \lambda \left(1 - \sum_{j=1}^n w_j^\:T w_j \right)
        + \nu \left(\log(t(c+2)) - \alpha^\:T w_c \right)
\end{aligned}
\end{equation}
Denote $\xi_j := \sum_{i=1}^m \epsilon_{i,j} {x}_i$. The Lagrangian is then maximized at $w_j = \xi_j/(2\lambda)$ for $j<c$ and $w_c = (\xi_c- \nu\alpha)/(2\lambda)$. The dual function then writes:
\begin{equation}
\begin{aligned}
    g(\lambda, \nu)
    = \nu \log(t(c+2))  +  \lambda  + \sum_{j=1}^{c-1} \frac{\| \xi_j \|^2_2}{4 \lambda} +\frac{\| \xi_c - \nu \alpha \|^2_2}{4 \lambda} 
    \ge \nu \log(t(c+2)) + \sqrt{ \sum_{j=1}^{c-1}\| \xi_j \|_2^2 + \| \xi_c - \nu \alpha \|_2^2 }
\end{aligned}
\end{equation}
By weak duality, we have that
 \begin{equation}
 \label{dual}
\begin{aligned}
    \hat{\mathfrak{R}}_m (\+F_t) 
    \le \frac{1}{m} \-E_\epsilon \left[
        \min_{\nu \ge 0} \left( 
            \nu \log(t(c+2)) + \sqrt{ \sum_{j=1}^{c-1}\| \xi_j \|_2^2 + \| \xi_c - \nu \alpha \|_2^2 }
        \right)
    \right]
\end{aligned}
\end{equation}

Assuming $t<1/(c+2)$ so that $\log(t(c+2))<0$. We can upper bound \eqref{dual} as 
\begin{equation}
    \frac{1}{m} \-E_\epsilon \left[
        \min_{\nu \ge 0} \left(
        \sqrt{ \sum_{j=1}^{c-1}\| \xi_j \|_2^2 + \| \xi_c - \nu \alpha \|_2^2 }
        \right)
    \right]
\end{equation}
The function $\sum_{j=1}^{c-1}\| \xi_j \|_2^2 + \| \xi_c - \nu \alpha \|_2^2$ is minimized at $\nu = 0$ if $\xi_c^\:T \alpha \le 0$ and $\nu = \xi_c^\:T \alpha /\|\alpha\|_2^2$ otherwise. Denote the event $\xi_c^\:T \alpha \le 0$ as $\+E$. By symmetry, we have that $\-P(\+E) = 1/2$ so that
\begin{equation}
    \frac{1}{m} \-E_\epsilon \left[
        \min_{\nu \ge 0} \left(
        \sqrt{ \sum_{j=1}^{c-1}\| \xi_j \|_2^2 + \| \xi_c - \nu \alpha \|_2^2 }
        \right)
    \right]
    = \frac{1}{2} \-E_\epsilon\left[ \sqrt{\sum_{j=1}^{c}\| \xi_j \|_2^2} \middle| \+E \right] 
        + \frac{1}{2} \-E_\epsilon \left[\sqrt{\sum_{j=1}^{c}\| \xi_j \|_2^2 - \frac{(\xi_c^\:T \alpha)^2}{\|\alpha\|_2^2}} \middle| \overline{\+E} \right]
\end{equation}
Again by symmetry, the quantity $(\xi_c^\:T \alpha)^2$ is independent of $\+E$. Therefore, by Jensen's inequality, we have that
\begin{equation}
    \begin{aligned}
        \-E_{S,\epsilon}
        \left[\sqrt{\sum_{j=1}^{c}\| \xi_j \|_2^2 - \frac{(\xi_c^\:T \alpha)^2}{\|\alpha\|_2^2}} \middle| \overline{\+E} \right]
        & \le \sqrt{
                \-E_{S,\epsilon} \left[
                    \sum_{j=1}^{c}\| \xi_j \|_2^2 - \frac{(\xi_c^\:T \alpha)^2}{\|\alpha\|_2^2}
                \right]
            } \\
        & \le \sqrt{
                cm - \-E_{S,\epsilon} \left[ \frac{(\xi_c^\:T \alpha)^2}{\|\alpha\|_2^2}
                \right]
            }  \\ 
        & = \sqrt{
                   cm - \frac{\operatorname{Var}(\xi_c^\:T \alpha)}{\|\alpha\|_2^2}
            }  \\ 
        & = \sqrt{
                   cm - m\frac{\sigma^2 \|\alpha\|_2^2+\|\alpha\|_2^4}{\|\alpha\|_2^2}
            } \\ 
        & = \sqrt{
                   (c-\sigma^2-\|\alpha\|_2^2)m
            }
    \end{aligned}
\end{equation}
Similarly, we can use Jensen's inequality to bound $\-E_{S,\epsilon}\left[ \sqrt{\sum_{j=1}^{c}\| \xi_j \|_2^2} \middle| \+E \right] \le \sqrt{cm}$. Putting these together, we have that
\begin{equation}
    \mathfrak{R}_m (\+F_t)
    =\-E_x[\hat{\mathfrak{R}}_m (\+F_t)]
    \le \frac{1}{2}\sqrt{\frac{c}{m}} +\frac{1}{2}\sqrt{\frac{c-\sigma^2-\|\alpha\|_2^2}{m}}
\end{equation}

\section{Proofs from Section 4}

\subsection{Proof of Propostion \ref{CCM-Rademacher}}

First, we show the Rademacher complexity of the singleton mapping is zero:
\begin{equation}
\begin{aligned}
    \mathfrak{R}_m(\{(x,y)\mapsto -\mu v(x,y)\}) 
    & = \frac{1}{m} \-E_{x, \epsilon} \left[ 
        \sum_{i=1}^{m} \sum_{y \in \+Y} -\epsilon_{i,y} \mu v(x_i,y)
    \right] \\
    & = \frac{1}{m} \-E_{x} \left[ 
        \sum_{i=1}^{m} \sum_{y \in \+Y} -\-E[\epsilon_{i,y}] \mu v(x_i,y)
    \right] \\
    & = 0
\end{aligned}
\end{equation}

Second, we use the linearity of Rademacher complexity to obtain the desired result. 
\begin{equation}
\begin{aligned}
    \mathfrak{R}_m(\+F^\mu) 
    & = \frac{1}{m} \-E_{x, \epsilon} \left[ \sup_{f \in \+F}
        \sum_{i=1}^{m} \sum_{y \in \+Y} \epsilon_{i,y} (f(x_i,y) - \mu v(x_i,y))
    \right] \\
    & = \frac{1}{m} \-E_{x, \epsilon} \left[ \sup_{f \in \+F}
        \sum_{i=1}^{m} \sum_{y \in \+Y} \epsilon_{i,y} f(x_i,y)
    \right] + \frac{1}{m} \-E_{x, \epsilon} \left[ 
        \sum_{i=1}^{m} \sum_{y \in \+Y} -\epsilon_{i,y} \mu v(x_i,y)
    \right]  \\
    & = \mathfrak{R}_m(\+F)  + \mathfrak{R}_m(\{(x,y)\mapsto -\mu v(x,y)\}) =  \mathfrak{R}_m(\+F) 
\end{aligned}
\end{equation}

\subsection{Proof of Proposition \ref{change-of-risk}}

\begin{enumerate}[label=(\alph*)]
\item
Given any scoring function $f$, let $Z_f^C(x) := \sum_{y \in C(x)} \exp(f(x,y))$ and $Z_f^{-C}(x) := \sum_{y \notin C(x)} \exp(f(x,y))$. We have 
    \begin{equation}
    \label{ce-derivative}
    \begin{aligned}
        \der{\mu} \Delta^\mu_{\ce} (f)
        & = \der{\mu} \-E\left[\log\frac{\exp(f(x,y_\ora)-\mu v(x,y_\ora))}{Z_f^C(x) + Z_f^{-C}(x)/\e^\mu}\right] \\ 
        & = \-E\left[ \der{\mu} \log\frac{\exp(f(x,y_\ora)-\mu v(x,y_\ora))}{Z_f^C(x) + Z_f^{-C}(x)/\e^\mu}\right] \\
        & = \-E\left[
           \frac{Z^{-C}_f(x)/\e^\mu}{Z_f^C(x) + Z_f^{-C}(x)/\e^\mu} - v(x,y_\ora)
        \right] \\
        & = V(f^\mu) - V_\ora
    \end{aligned}
    \end{equation}
Moreover, 
\begin{equation}
    \begin{aligned}
        \der{\mu} V(f^\mu)
        = & \-E\left[ \der{\mu} \frac{Z_{f^\mu}^{-C}(x)}{Z_{f^\mu}(x)} \right]\\ 
        = & \-E\left[ \frac{Z_{f^\mu}(x)(-Z_{f^\mu}^C(x)) + (Z_{f^\mu}^C(x))^2}{\left(Z_{f^\mu}(x)\right)^2} \right] \\ 
        = & \-E\left[  \-P_{f^\mu}^2(-C) -  \-P_{f^\mu}(-C)  \right]
    \end{aligned}
\end{equation}
which is negative and bounded, implying $V(f^\mu) - V_\ora$ is decreasing and Lipschitz with $\mu$. Therefore, there is a $\mu > 0$ such that $R_\ce(f^\mu) < R_\ce(f)$ if and only if the derivative is positive at $\mu = 0$, i.e., $V(f) > V_\ora$.

\item 
By \eqref{ce-derivative},
\begin{equation}
\begin{aligned}
    \Delta^\mu_{\ce} (f) 
    & = \int^\mu_0 \left(V(f^t) - V_\ora\right) \dd t \\ 
    & = \-E\left[ \int^\mu_0 
        \frac{Z^{-C}_f(x)/\e^t}{Z_f^C(x) + Z_f^{-C}(x)/\e^t} \dd t
     \right] - \mu V_\ora \\ 
    & \ge \-E\left[ \int^\mu_0 
    \frac{Z^{-C}_f(x)/\e^t}{Z_f^C(x) + Z_f^{-C}(x)} \dd t
    \right] - \mu V_\ora \\ 
    & = (1-\e^{-\mu})  \-E\left[
    \frac{Z^{-C}_f(x)}{Z_f^C(x) + Z_f^{-C}(x)}
    \right] - \mu V_\ora \\ 
    & = (1-\e^{-\mu}) V(f) - \mu V_\ora 
\end{aligned}
\end{equation}

\item
If $V_\ora=0$, we have 
\begin{equation}
\begin{aligned}
    \Delta^\infty_{\ce} (f) 
    & = \int^\infty_0 \-E\left[
        \frac{Z^{-C}_f(x)/\e^t}{Z_f^C(x) + Z_f^{-C}(x)/\e^t}\right] \dd t \\ 
    & = \-E\left[ \int^\infty_0 
        \frac{Z^{-C}_f(x)/\e^t}{Z_f^C(x) + Z_f^{-C}(x)/\e^t} \dd t \right] \\ 
    & = \-E\left[
        \log \left(\frac{Z_f^C(x) + Z_f^{-C}}{Z_f^C} \right)
    \right] \\ 
    & = V_\ce(f)
\end{aligned}
\end{equation}

\end{enumerate}

\subsection{Proof of Corollary \ref{mu-upper-bound}}

Using Proposition \ref{change-of-risk} (b), this result follows by solving the following equation
\begin{equation}
    (1-\e^{-\mu}) V(f) - \mu V_\ora 
    \ge 0
\end{equation}

It is known that the solution to the inequaltiy $u \le a + b\e^{c u}$ of $u$ is $u \le a-\frac{1}{c}W(-bc\e^{ac})$. Substituting $a=\eta=V(f)/V_\ora=-b$ and $c=-1$ yields the desired result:
\begin{equation}
    \begin{aligned}
        \mu 
        \le W(-\eta/\e^\eta)+\eta
    \end{aligned}
\end{equation}
where the RHS is positive only when $\eta>1$. A plot of this solution as a function of $\eta$ is presented below in Figure \ref{solution}.

\begin{figure}[h]
    \centering    \includegraphics[width=0.36\textwidth]{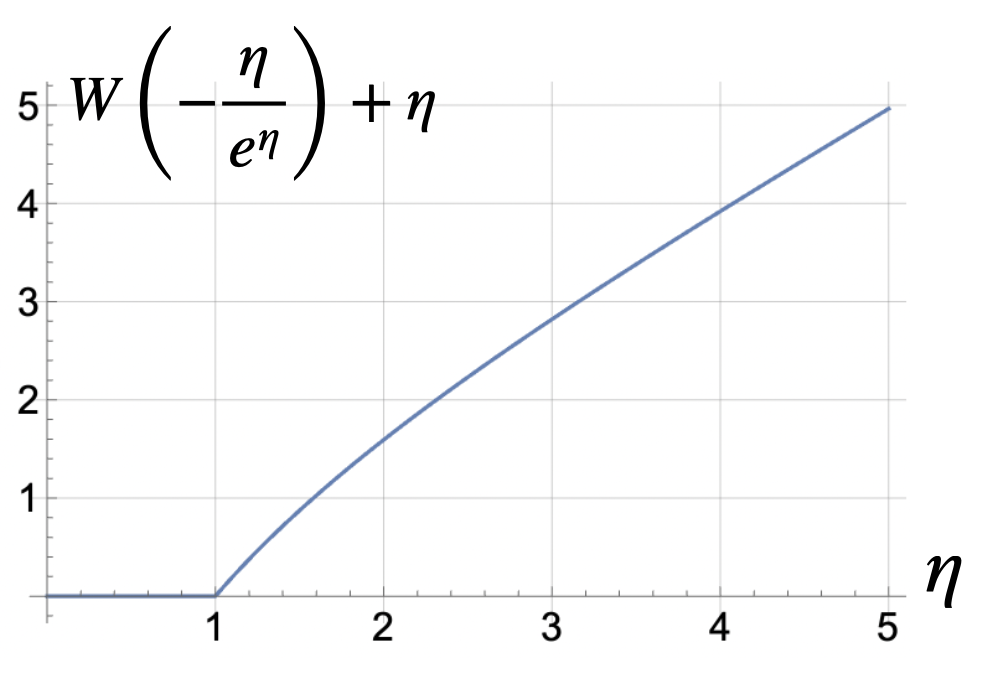}
    \caption{Choice of $\mu$ as a function of $\eta=V(f)/V_\ora$.}
    \label{solution}
\end{figure}

\subsection{Proof of Proposition \ref{on-training}}

This claim follows from the fact that $R_\ce(f^\infty)=R_\ce(f)-V_\ce(f)$ from Proposition \ref{change-of-risk} (c). 

For equation \eqref{post-vs-on}, the first inequality follows from the optimality of $\on$. For the second inequality, by definition we have 
\begin{equation}
\begin{aligned}
    & R_\ce(\post^\infty) + V_\ce(\post) = R_\ce(\post) 
    \le R_\ce(\on) \\ 
    \Rightarrow &  R_\ce(\post^\infty) \le  R_\ce(\on) - V_\ce(\post) \le  R_\ce(\on)  - \min_{f\in\+F} V_\ce(f)
\end{aligned}
\end{equation}

\section{Analysis for Hinge Loss and $\ell^1$ Loss}
\label{alternative-loss}

\subsection{Hinge Loss}
\label{hinge-Loss}

The \emph{margin} of a scoring function $f$ at a sample $(x,y_\ora)$ is defined as 
\begin{equation}
\label{margin}
\begin{aligned}
    m(x,y_\ora, f)
    := \max_{y\in\+Y} \left\{f(x,y)\right\} - f(x,y_\ora)
\end{aligned}
\end{equation}
We denote its expectation as $M(f) = \-E[m(x,y_\ora,f)]$.

Given a loss function $\ell:\+Y\times\+Y \rightarrow \-R$, the structured hinge loss \citep{LondonHuGe16, MLWS19} is defined as the margin of the \emph{loss augmented} scoring function $f+\ell: (x,y)\mapsto f(x,y) + \ell(y, y_\ora)$. Namely,
\begin{equation}
\label{hinge}
\begin{aligned}
    L_\text{hinge} (x,y_\ora, f)
    := m(x,y_\ora, f+\ell)
\end{aligned}
\end{equation}
Therefore, we can study the impact of constrained inference on the hinge loss via the impact on the margin. Let $\Delta_\mathrm{margin}^\mu(f) = M(f) - M(f^\mu)$. We present the following result.

\begin{theorem}[Change of Margin]
\label{change-of-risk-hinge}
    The following results hold:
    \begin{enumerate}[label=(\alph*)]
        \item 
        For any fixed model $f$, there exists an $\mu_0 > 0$ such that $M(f^\mu) \le M(f)$ only if 
        \begin{equation}
        \begin{aligned}
            V_{01}(f) > V_\ora
        \end{aligned}
        \end{equation}
        where $V_{01}(f)$ is the zero-one style violation defined as $\-E[\-1\{\argmax_{y \in \+Y}f(x,y) \ne y_\ora\}]$.

        \item 
        In particular, if the constraint is noise-free, we have 
        \begin{equation}
        \begin{aligned}
           \Delta^\infty_\mathrm{margin}(f)
            & = \-E\left[ \max_{y \in \+Y} f(x,y) - \max_{y\in C(x)} f(x,y) \right] \\
            & = \-E\left[ \left(\max_{y \notin C(x)} f(x,y) - \max_{y\in C(x)} f(x,y)\right)_+ \right]
        \end{aligned}
        \end{equation}
    \end{enumerate}
\end{theorem}

\begin{proof}
\begin{enumerate}[label=(\alph*)]
    \item 
    The derivative of the change of the margin is
    \begin{equation}
    \begin{aligned}
        \der{\mu} \Delta^\mu_\mathrm{margin}(f) =
        -\der{\mu} M(f^\mu)
        & = - \der{\mu} \-E \left[
            \max_{y \in \+Y} \{ f(x,y) - \mu v(x,y) \} - f(x,y_\ora) + \mu v (x,y_\ora) 
        \right] \\ 
        & = \-E[v(x,y_{f^\mu}) - v(x,y_\ora)]
    \end{aligned}
    \end{equation}
    where $y_{f^\mu}:= \argmax_{y \in \+Y} \{ f(x,y) - \mu v(x,y)\}$ is the argmax inference output of CCM. Moreover, this derivative is non-increasing with $\mu$. Therefore, a necessary condition for CCM to reduce the margin is 
    \begin{equation}
        \-E[v(x,y_f)] = V_{01}(f) 
        > V_\ora
    \end{equation}

    \item 
    This follows directly by taking the difference between $M(f)$ and $M(f^\infty)$.

\end{enumerate}
\end{proof}

\begin{remark}
Due to the discontinuous nature of the argmax inference, the function $v(x,y_{f^\mu})$ is in general not continuous with $\mu$. On the other hand, if we assume $\mu \mapsto \-E[v(x,y_{f^\mu})]$ is Lipschitz continuous, the condition proposed in (a) is also sufficient, as in the analysis for cross-entropy.

The impact of constrained inference on the hinge loss can be investigated by substituting $f$ by $f+\ell$. For example, a sufficient for improving the average hinge loss will be $V_{01}(f+\ell) > V_\ora$.

The quantity $\left(\max_{y \notin C(x)} f(x,y) - \max_{y\in C(x)} f(x,y)\right)_+$ is closely related to the \emph{integrality loss} defined in \citet{MLWS19}. It is a hinge-stye surrogate loss function for the zero-one style violation function of $f$ with argmax inference:
\begin{equation}
\begin{aligned}
    \-P\left\{
        \max_{y \notin C(x)} f(x,y) - \max_{y\in C(x)} f(x,y) 
        \ge 0
    \right\}
    = V_{01}(f)
\end{aligned}
\end{equation}
\end{remark}

\subsection{$\ell^1$ Loss}
\label{l1-Loss}

To facilitate our discussion, we first present the following lemmas that will be useful in this section.

\begin{lemma}[Gradients of CCM]
\label{l1-gradient}
    For any constraint $C$ we have the following:
    \begin{enumerate}
        \item 
        The derivative of the predicted probability is 
        \begin{equation}
        \begin{aligned}
            \der{\mu} \-P_{f^\mu}(y|x) 
            = \-P_{f^\mu}(y) \left(\-P_{f^\mu}(-C|x) - v(x,y)\right)
        \end{aligned}
        \end{equation}

        \item 
        The second order derivative of the probability is 
        \begin{equation}
        \begin{aligned}
            \der{\mu} \-P_{f^\mu}(-C|x) 
            = \-P_{f^\mu}(y|x) \left(
                \left( \-P_{f^\mu}(-C|x) - v(x,y)\right)^2 + \-P_{f^\mu}^2(-C|x) -  \-P_{f^\mu}(-C|x) 
                \right)
        \end{aligned}
        \end{equation}
    \end{enumerate}
\end{lemma}

\begin{proof}

Recall that given any scoring function $f$, we denote $$Z_f^C(x) := \sum_{y \in C(x)} \exp(f(x,y))$$ and $$Z_f^{-C}(x) := \sum_{y \notin C(x)} \exp(f(x,y))$$
We also let $Z_f(x) = Z_f^C(x) + Z_f^{-C}(x)$.

\begin{enumerate}[label=(\alph*)]

    \item 
    The pointwise derivative of CCM's $l^1$ risk with respect to $\mu$ is then 
    \begin{equation}
    \begin{aligned}
        \der{\mu} \-P_{f^\mu}(y|x)
        = & \der{\mu} \frac{\e^{f(x,y) - \mu v(x,y)}}{Z_{f^\mu}(x)} \\ 
        = & \frac{1}{\left(Z_{f^\mu}(x)\right)^2} \left( Z_{f^\mu}(x) (-v(x,y) \e^{f(x,y) - \mu v(x,y)})  + Z_{f^\mu}^{-C}(x) \e^{f(x,y) - \mu v(x,y)} \right)\\ 
        = & \-P_{f^\mu}(y) \left(\-P_{f^\mu}(-C) - v(x,y)\right)
    \end{aligned}
    \end{equation}
    where the second equality follows from the fact that $\der{\mu} Z_{f^\mu}(x) = -Z_{f^\mu}^{-C}(x)$.

    \item
    Based on (a),
    \begin{equation}
        \begin{aligned}
            \frac{\dd^2}{\dd^2 \mu} \-P_{f^\mu}(y|x) 
            & = \left(\-P_{f^\mu}(y) \left(\-P_{f^\mu}(-C) - v(x,y)\right)\right)\left(\-P_{f^\mu}(-C) - v(x,y)\right) \\ 
            & \quad + \-P_{f^\mu}(y) \left(\-P_{f^\mu}^2(-C) -  \-P_{f^\mu}(-C)\right)\\
            & = \-P_{f^\mu}(y|x) \left(
                \left( \-P_{f^\mu}(-C|x) - v(x,y)\right)^2 + \-P_{f^\mu}^2(-C|x) -  \-P_{f^\mu}(-C|x) 
                \right)
        \end{aligned}
        \end{equation}
\end{enumerate}
\end{proof}

Now we discuss the change in $\ell^1$ risk that is defined as $\Delta^\mu(f):=R(f)-R(f^\mu)$.

\begin{theorem}[Change of $\ell^1$ Risk]
\label{change-of-risk-l1}
    The following results hold:
    \begin{enumerate}[label=(\alph*)]
        \item 
        For any fixed model $f$, there exists an $\mu_0 > 0$ such that $R(f^\mu) < R(f)$ if 
        \begin{equation}
        \begin{aligned}
            \-E[\-P_f(y_\ora)\-P_f(-C)] 
            > \-E[\-P_f(y_\ora)v(x,y_\ora)]
        \end{aligned}
        \end{equation}

        \item 
        The change of risk can be lower bounded by
        \begin{equation}
        \label{risk-bound-ce-loss}
        \begin{aligned}
            \Delta^\mu(f)
            \ge \frac{1-\e^{-2\mu}}{2}\-E_x[\-P_f(y_\ora)\-P_f(-C)] - \mu V_\ora
        \end{aligned}
        \end{equation}

        \item 
        In particular, if the constraint is noise-free, we have 
        \begin{equation}
        \begin{aligned}
            \Delta^\infty(f)
            \ge \-E_x[\-P_f(y_\ora)\-P_f(-C)] 
        \end{aligned}
        \end{equation}
    \end{enumerate}
\end{theorem}

\begin{proof}
\begin{enumerate}[label=(\alph*)]
    \item 
    From Lemma \ref{l1-gradient} (a) we know the derivative of the risk with respect to $\mu$ at $\mu=0$ is 
    \begin{equation}
    \begin{aligned}
        \-E[\-P_f(y_\ora)\-P_f(-C)] - \-E[\-P_f(y_\ora)v(x,y_\ora)]
    \end{aligned}
    \end{equation}
    Further, Lemma \ref{l1-gradient} (b) implies this derivative is Lipschitz with respect to $\mu$ since for any $\mu$,
    \begin{equation}
    \begin{aligned}
        \left| \-P_{f^\mu}(y|x) \left(
                \left( \-P_{f}(-C|x) - v(x,y)\right)^2 + \-P_{f^\mu}^2(-C|x) -  \-P_{f^\mu}(-C|x) 
                \right) \right| 
            \le 1
    \end{aligned}
    \end{equation} 
    Therefore, a sufficient condition for the existence of an $\mu_0 > 0$ such that $R(f^\mu) < R(f)$ is that $\-E[\-P_f(y_\ora)\-P_f(-C)] > \-E[\-P_f(y_\ora)v(x,y_\ora)]$.

    \item 
    First, we note for any $y$ and $\mu$ that
    \begin{equation}
    \begin{aligned}
        \-P_{f^\mu}(y)\-P_{f^\mu}(-C)
        & = \frac{\e^{f(x,y)-\mu v(x,y)} Z_f^{-C}(x)/\e^\mu}{\left(Z_{f^\mu}(x)\right)^2} \\
        & \ge \frac{\e^{f(x,y)-\mu v(x,y)} Z_f^{-C}(x)/\e^\mu}{\left(Z_{f}(x)\right)^2} \\ 
        & \ge \frac{\e^{f(x,y)-\mu } Z_f^{-C}(x)/\e^\mu}{\left(Z_{f}(x)\right)^2} \\ 
        & = \-P_f(y)\-P_f(-C)\e^{-2\mu}
    \end{aligned}
    \end{equation}
    Also, 
    \begin{equation}
    \begin{aligned}
        \-E[\-P_f(y_\ora)v(x,y_\ora)] 
        \le \-E[v(x,y_\ora)] 
        = V_\ora
    \end{aligned}
    \end{equation}
    Integrating the derivative gives 
    \begin{equation}
    \begin{aligned}
        \Delta^\mu(f)
        & \ge \int^\mu_0 \-E\left[
            \-P_f(y_\ora)\-P_f(-C)\e^{-2t} - V_\ora
        \right] \dd t \\
        & = \frac{1-\e^{-2\mu}}{2}\-E_x[\-P_f(y_\ora)\-P_f(-C)] - \mu V_\ora
    \end{aligned}
    \end{equation}

    \item 
    With noise-free constraints,
    \begin{equation}
        \begin{aligned}
            \-P_{f^\mu}(y_\ora)\-P_{f^\mu}(-C)
            & = \frac{\e^{f(x,y_\ora)} Z_f^{-C}(x)/\e^\mu}{\left(Z_{f^\mu}(x)\right)^2} \\
            & \ge \frac{\e^{f(x,y_\ora)} Z_f^{-C}(x)/\e^\mu}{\left(Z_{f}(x)\right)^2} \\ 
            & = \-P_f(y_\ora)\-P_f(-C)\e^{-\mu}
    \end{aligned}
    \end{equation}
    Integrating both sides gives 
    \begin{equation}
        \begin{aligned}
            \Delta^\mu(f)
            & \ge \int^\mu_0 \-E\left[
                \-P_f(y_\ora)\-P_f(-C)\e^{-t} 
            \right] \dd t \\
            & = \-E_x[\-P_f(y_\ora)\-P_f(-C)]
    \end{aligned}
    \end{equation}
\end{enumerate}
\end{proof}

The term $\-E_x[\-P_f(y_\ora)\-P_f(-C)]$ plays a key role in these results, and it measures the average violation of the model $f$, weighted by the model's confidence of the true label. The first result shows that if this weighted average violation is larger than that of the true data distribution, then CCM is helpful. The last result shows that a model with a larger weighted violation obtains more benefits from strictly constrained inference.

\section{Proofs from Section 5}

\subsection{Proof of Theorem \ref{combo-help}}

Recall $f_\star^\mu = \argmin_{g\in\+F^\mu} R_\ce(g) + \rho V_\ce(g)$ is the optimal CCM for the regularized surrogate objective and $\post$ is the cross entropy risk minimizer in $\+F$. According to our notation, $\post^\mu$ is the constrained model with base model $\post$. 
By this definition, we have 
\begin{equation}
\begin{aligned}
    R_\ce(f_\star^\mu ) +\rho V_\ce(f_\star^\mu)
    \le R_\ce(\post^\mu) +\rho V_\ce(\post^\mu)
\end{aligned}
\end{equation}
Therefore,
\begin{equation}
\label{sufficient-condition-combo}
\begin{aligned}
    R_\ce(f_\star^\mu) 
    & \le R_\ce(\post^\mu) + \rho (V_\ce(\post^\mu) - V_\ce(f_\infty^\mu)) \\
    & \le R_\ce(\post^\mu) + \rho V_\ce(\post^\mu) \\ 
    & \le R_\ce(\post) - \Delta_\ce^\mu(\post) + \rho V_\ce(\post^\mu) 
\end{aligned}
\end{equation}
Therefore, a sufficient condition for $R_\ce(f_\star^\mu) \le R_\ce(\post)$ is that $ \rho V_\ce(\post^\mu) < \Delta_\ce^\mu(\post)$. Furthermore, recall for any scoring function $f$, we define $Z_f^C(x) := \sum_{y \in C(x)} \exp(f(x,y))$ and $Z_f^{-C}(x) := \sum_{y \notin C(x)} \exp(f(x,y))$. We then have
\begin{equation}
\begin{aligned}
    V_\ce(f) - V_\ce(f^\mu) 
    & = \-E\left[
        -\log \left( \frac{Z_f^C(x)}{Z_f^C(x) + Z_f^{-C}(x)} \right)
    \right] - \-E\left[
        -\log \left( \frac{Z_f^C(x)}{Z_f^C(x) + Z_f^{-C}(x)/\e^\mu}\right)
    \right] \\ 
    & = \-E\left[
        -\log \left( \frac{Z_f^C(x) + Z_f^{-C}(x)/\e^\mu}{Z_f^C(x) + Z_f^{-C}(x)} \right)
    \right] \\ 
    & = \int^\mu_0 \-E\left[
        \frac{Z^{-C}_f(x)/\e^t}{Z_f^C(x) + Z_f^{-C}(x)/\e^t}\right] \dd t \\ 
    & = \Delta^\mu_\ce(f) + \mu V_\ora \quad \quad \quad \quad  \text{(compare to equation \eqref{ce-derivative})}
\end{aligned}
\end{equation}
Therefore, $\Delta^\mu_\ce(\post) =  V_\ce(\post) - V_\ce(\post^\mu) -  \mu V_\ora$. So, the sufficient condition can be reformulated as 
\begin{equation}
\begin{aligned}
    \rho 
    < \frac{V_\ce(\post) - V_\ce(\post^\mu) - \mu V_\ora }{V_\ce(\post^\mu)}
\end{aligned}
\end{equation}

\subsection{Proof of Theorem \ref{combo-not-help}}

We have seen in Theorem \ref{change-of-risk} that for any scoring function $f$, there is a $\mu > 0$ such that $R_\ce(f^\mu) < R_\ce(f)$ if and only if $V(f) \ge V_\ora$. On the other hand, we know from Lemma \ref{f-rho-violation} that
\begin{equation}
\begin{aligned}
    V(f_\rho) 
    \le V(f_\infty) + \frac{1}{\rho}
\end{aligned}
\end{equation}
Therefore, if
\begin{equation}
\begin{aligned}
    \rho 
    \ge \frac{1}{V_\ora - V(f_\infty)}
\end{aligned}
\end{equation}
we must have $V(f_\rho) \le V_\ora$, which implies there is no $\mu > 0$ such that $R_\ce((f_\rho)^\mu) < R_\ce(f_\rho)$.